%% file: vertical-main.tex
\RequirePackage{tikz}
\documentclass[pdflatex,sn-basic]{sn-jnl}

\jyear{2021}%
\input{preambel}
\theoremstyle{thmstyleone}%
\newtheorem{proposition}{Proposition}%

\newtheorem{constraint}{Constraint}%
\crefname{constraint}{constraint}{constraints}
\theoremstyle{thmstyletwo}%

\theoremstyle{thmstylethree}%

\begin{document}

\title[Federated singular value decomposition]{Federated singular value decomposition for high dimensional data}


\author*[1]{\fnm{Anne} \sur{Hartebrodt}}\email{hartebrodt@imada.sdu.dk}

\author[1]{\fnm{Richard} \sur{R\"ottger}}\email{roettger@imada.sdu.dk}
\equalcont{Joint last authors}

\author[2]{\fnm{David B.} \sur{Blumenthal}}\email{david.b.blumenthal@fau.de}
\equalcont{Joint last authors}

\affil*[1]{\orgdiv{Department of Mathematics and Computer Science}, \orgname{University of Southern Denmark}, \orgaddress{\street{Campusvej 55}, \city{Odense}, \postcode{5230}, , \country{Denmark}}}

\affil[2]{\orgdiv{Department Artificial Intelligence in Biomedical Engineering (AIBE)}, \orgname{Friedrich-Alexander University Erlangen-Nürnberg (FAU)}, \orgaddress{\street{Konrad-Zuse-Str. 3/5}, \postcode{91052} \city{Erlangen}, \country{Germany}}}

\affil[a]{ORCID:0000-0002-9172-3137}
\affil[b]{ORCID:0000-0003-4490-5947}
\affil[c]{ORCID:0000-0001-8651-750X}



\abstract{
	Federated learning (FL) is emerging as a privacy-aware alternative to classical cloud-based machine learning. In FL, the sensitive data remains in data silos and only aggregated parameters are exchanged. Hospitals and research institutions which are not willing to share their data can join a federated study without breaching confidentiality. In addition to the extreme sensitivity of biomedical data, the high dimensionality poses a challenge in the context of federated genome-wide association studies (GWAS). In this article, we present a federated singular value decomposition (SVD) algorithm, suitable for the privacy-related and computational requirements of GWAS. Notably, the algorithm has a transmission cost independent of the number of samples and is only weakly dependent on the number of features, because the singular vectors associated with the samples are never exchanged and the vectors associated with the features only for a fixed number of iterations. Although motivated by GWAS, the algorithm is generically applicable for both horizontally and vertically partitioned data.
}

\maketitle

\include{introduction}
\include{preliminaries}
\include{related}

\include{algorithm}
\include{leakage}

\include{experiments}

\include{conclusions}

\section*{Funding}
The FeatureCloud project has received funding from the European Union’s Horizon 2020 research and innovation programme under grant agreement No 826078. This publication reflects only the authors’ view and the European Commission is not responsible for any use that may be made of the information it contains.

\section*{Competing interests}
The authors have no relevant financial or non-financial interests to disclose.

\bibliography{references}


\end{document}

%% file: preambel.tex
\usepackage{amsfonts} 					
\usepackage{amsmath}	
\usepackage{hyperref}
\usepackage{cleveref}
\usepackage{stmaryrd}					
\usepackage{mathtools} 
\usepackage{nicefrac} 
\usepackage{wasysym}					
\usepackage{url} 
\usepackage{microtype} 
\usepackage{xspace} 
\usepackage{csquotes} 
\usepackage{enumitem}
\usepackage{siunitx}
\usepackage{booktabs}
\usepackage{balance}
\usepackage{xcolor}

\usetikzlibrary{calc,backgrounds,arrows,decorations.markings,decorations.pathreplacing,matrix,graphs,positioning,fit,scopes}
\usepackage{ifthen}
\usepackage{xstring}
\usepackage{calc}
\usepackage{pgfopts}

\SetKwInput{KwInput}{Input}                
\SetKwInput{KwOutput}{Output}              

\SetCommentSty{mycommfont}

\DeclareMathOperator{\diag}{diag}
\DeclareMathOperator{\thespan}{span}
\renewcommand{\A}{\ensuremath{\mathbf{A}}\xspace}
\newcommand{\D}{\ensuremath{\mathbf{D}}\xspace}

\newcommand{\G}{\ensuremath{\mathbf{G}}\xspace}
\newcommand{\g}{\ensuremath{\mathbf{g}}\xspace}
\renewcommand{\H}{\ensuremath{\mathbf{H}}\xspace}

\newcommand{\U}{\ensuremath{\mathbf{U}}\xspace}
\newcommand{\V}{\ensuremath{\mathbf{V}}\xspace}
\renewcommand{\u}{\ensuremath{\mathbf{u}}\xspace}
\renewcommand{\v}{\ensuremath{\mathbf{v}}\xspace}
\renewcommand{\y}{\ensuremath{\mathbf{y}}\xspace}
\renewcommand{\x}{\ensuremath{\mathbf{x}}\xspace}
\newcommand{\M}{\ensuremath{\mathbf{M}}\xspace}
\renewcommand{\P}{\ensuremath{\mathbf{P}}\xspace}
\newcommand{\K}{\ensuremath{\mathbf{K}}\xspace}

\newcommand{\orthonormalize}{\ensuremath{\operatorname{\mathtt{orthonormalize}}}\xspace}
\newcommand{\fedOrthonormalize}{\ensuremath{\operatorname{\mathtt{federated-gram-schmidt}}}\xspace}

\newcommand{\stackVert}{\ensuremath{\operatorname{\mathtt{stack-vertically}}}\xspace}
\newcommand{\fedSub}{\ensuremath{\operatorname{\mathtt{federated-subspace-iteration}}}\xspace}
\newcommand{\initApprox}{\ensuremath{\operatorname{\mathtt{init-approximative}}}\xspace}

\newcommand{\SVD}{\ensuremath{\operatorname{\mathtt{singular-value-decomposition}}}\xspace}
\newcommand{\ie}{i.\,e.\@\xspace}
\newcommand{\eg}{e.\,g.\@\xspace}

\newcommand{\wrt}{w.\,r.\,t.\@\xspace}

\newcommand{\cf}{cf.\@\xspace}

\newcommand{\FEDGS}{\textsf{FED-GS}\xspace}

\newcommand{\GUO}{\textsf{GUO}\xspace}
\newcommand{\FEDRI}{\textsf{RI-FULL}\xspace}
\newcommand{\FEDAI}{\textsf{AI-FULL}\xspace}
\newcommand{\FEDRANDRI}{\textsf{RANDOMIZED}\xspace}
\newcommand{\AIONLY}{\textsf{AI-ONLY}\xspace}

%% file: introduction.tex

\section{Introduction}
Federated learning (FL) has recently gained attention as a privacy-aware alternative to centralized computation. 
Unlike in centralized machine learning where the data is consolidated at a central server and a model is calculated on the combined data, in FL, the data remains at the data owners machine \citep{Mothukuri_2021}. Instead of the data, only model parameters are sent to the (untrusted) aggregator which combines the local models into a global model. No raw data is exchanged in FL. See \Cref{fig:federated-learning-schematic} for a schematic comparison of centralized (cloud) learning and FL. In the cloud-based approach, data contributors send their private data to a central server where the model is computed and thereby lose agency over it.

\begin{figure}[htb]
\centering
\includegraphics[width=0.45\linewidth]{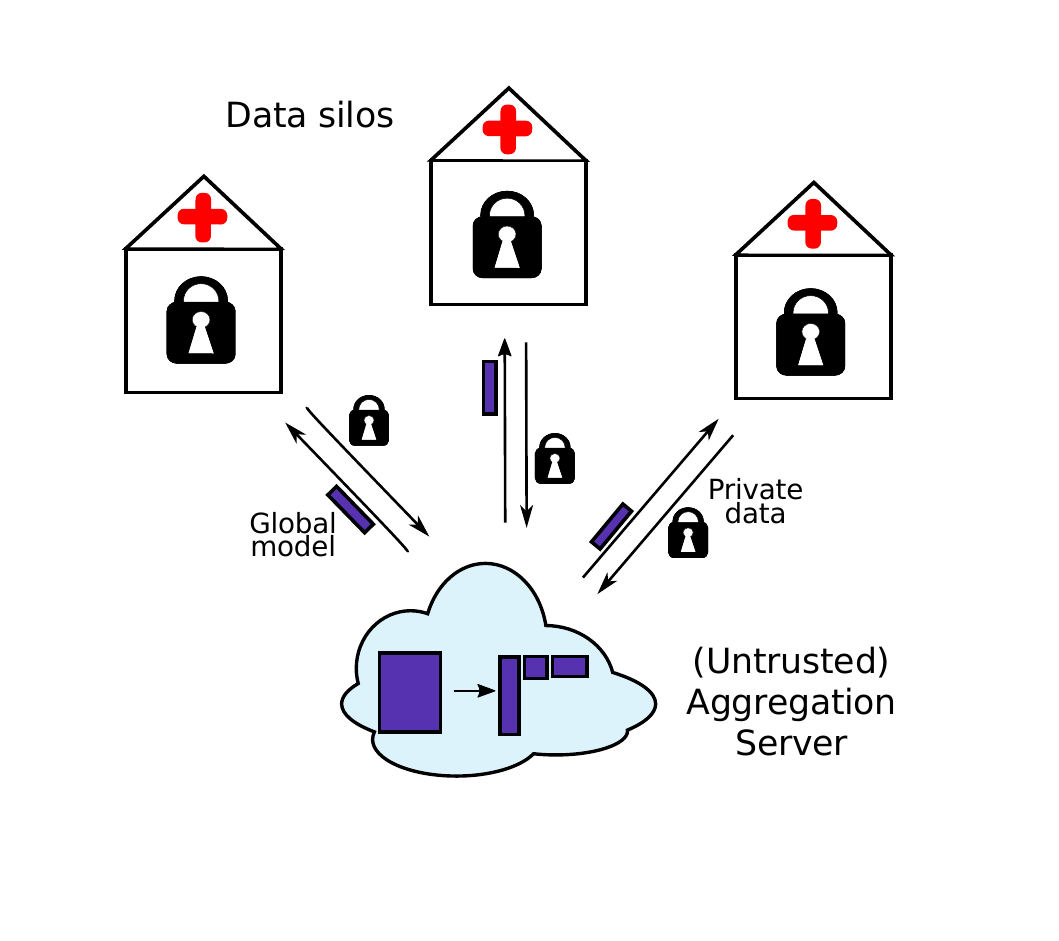}
\hspace{0.05\linewidth}
\includegraphics[width=0.45\linewidth]{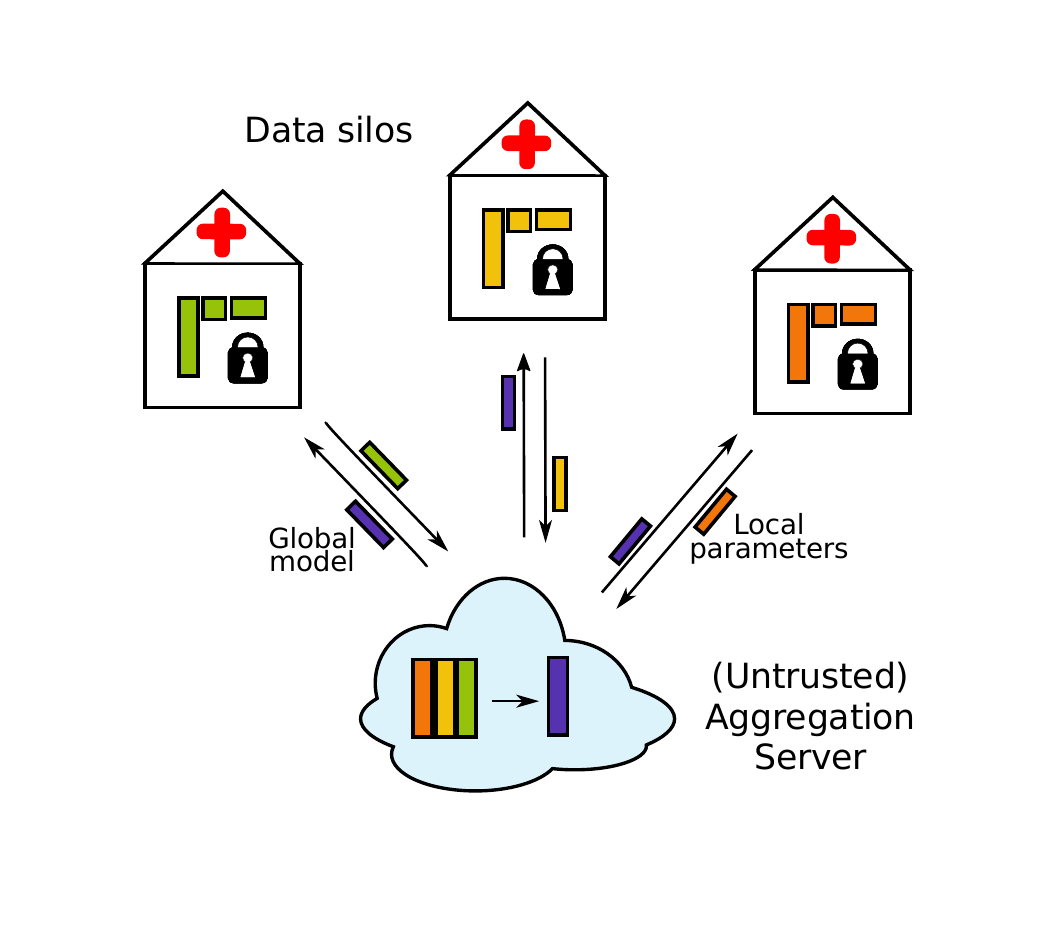}
\caption{Schematic comparison of traditional cloud base approaches (left) and federated learning (right).}
\label{fig:federated-learning-schematic}
\end{figure}

FL is subdivided in cross-device and cross-silo FL. Cross-device FL assumes a high number of devices, such as mobile phones or sensors with limited compute power connected in a dynamic fashion, meaning clients are expected join and drop out during the learning process. Cross-silo FL has a lower number of participants which hold a larger amount of data, have higher compute power and are connected in a more static fashion. Clients are not expected to join and drop out during the learning process randomly \citep{Kairouz2021}.

An attractive application case for cross-silo FL are genome-wide association studies (GWAS), which investigate the relationship of genetic variation with phenotypic traits on large cohorts \citep{Visscher2017, Tam2019}. Since genetic data are extremely sensitive, data holders cannot make it publicly available. The practical feasibility of using FL for GWAS has been demonstrated recently \citep{splink2020, Cho2018}.

Since GWAS are often done on populations of mixed ancestry, cryptic population confounders should be controlled for before associating the genetic variants to the phenotypic trait of interest. The standard way for doing this is to compute the leading eigenvectors of the sample covariance matrix via principal component analysis (PCA), and to include these eigenvectors as confounding variables to the models used for the association tests \citep{Price2006, Galinsky2016}.

For federated GWAS, a PCA algorithm for vertically partitioned data is required for computing the eigenvectors (see \Cref{sec:prelim} for a detailed explanation). Although a few such algorithms are available \citep{Kargupta2001,Qi2003,Guo2012,Wu2018}, none of them is suitable for federated GWAS. More precisely, the algorithms reviewed by \cite{Wu2018} use client-to-client communication and are therefore unsuitable for the star-like FL architectures used in GWAS, where relatively few data holders collaborate in a static setting. The algorithms presented by \cite{Kargupta2001} and \cite{Qi2003} rely on estimating a proxy covariance matrix and hence do not scale to large GWAS datasets, which often contain genetic variation data for hundreds of thousands of individuals. One of the few covariance free PCA algorithm suitable for a star-like architecture has been presented by \cite{Guo2012}. However, this algorithm broadcasts the complete first $k-1$ sample eigenvectors to the aggregator, which constitutes a privacy leakage that should be avoided in federated GWAS \citep{sigir2021}. \cite{Cho2018} present a secure multiparty protocol for GWAS which includes PCA relying on householder reflections. The protocol includes three external parties and potential physical shipping of data. The setup is fundamentally different: the data holders are individuals who only have access to one record. They create secret shares which are processed by two computing parties.

In previous algorithms, including algorithms designed for horizontally partitioned data, such as described by \cite{Balcan_2014}, the exchanged parameters scale with the number of genetic variants (features) in the data set as the feature eigenvectors are exchanged. At the scale of GWAS with several million genetic variants this is another challenge for the existing algorithms. Furthermore, due to the iterative nature of the algorithm, the process is prone to information leakage, a problem previously not investigated. More precisely, the feature eigenvector updates exchanged during the learning process can be used to compute the feature-covariance matrix given a sufficient number of iterations. This makes the algorithm equivalent with algorithms exchanging the entire covariance matrix in terms of disclosed information. The feature covariance matrix is a summary statistic over all samples, but due to its size contains a high amount of information and can be used to generate realistically looking samples. Therefore, the communication of the entire feature eigenvectors should also be avoided as far as possible.

Extrapolating from the shortcomings of existing approaches, we can state that, for federated GWAS, a PCA algorithm for vertically partitioned data is required that combines the following properties:
\begin{itemize}
\item The algorithm should be suitable for a star-like FL architecture, \ie, require only client-to-aggregator but no client-to-client communication.
\item The algorithm should not rely on computing or approximating the covariance matrix.
\item The algorithm should be communication efficient.
\item The algorithm should avoid the communication of the sample eigenvectors and reduce the communication of the feature eigenvectors.
\end{itemize}

In this paper, we present the first algorithm that combines all of these desirable properties and can hence be used for federated GWAS (and all other applications where these properties are required). We prove that our algorithm is equivalent to centralized vertical subspace iteration \citep{Halko_2010}\,---\,a state-of-the-art centralized, covariance-free SVD algorithm\,---\,and therefore generically applicable to any kind of data. Thereby, we show that the notion of \enquote{horizontally} and \enquote{vertically} partitioned data are irrelevant for SVD. Furthermore, we apply two strategies to make the algorithm more communication efficient, both in terms of communication rounds and transmitted data volume. More specifically, we employ approximate PCA \citep{Balcan_2014} and randomized PCA \citep{Halko_2010}. We show in an empirical evaluation that the eigenvectors computed by our approaches converge to the centrally computed eigenvectors after sufficiently many iterations. In sum, the article contains the following main contributions:

\begin{itemize}
    \item We present the first federated PCA algorithm for vertically partitioned data which meets the requirements that apply in federated GWAS settings.
    \item We prove that our algorithm is equivalent to centralized power iteration and show that it exhibits an excellent convergence behavior in practice.
    \item This algorithm is generically applicable for federated singular value decomposition on both \enquote{horizontally} and \enquote{vertically} partitioned data.
\end{itemize}

This article is an extended and consolidated version of a previous conference publication \citep{Hartebrodt2021} with the following additional contributions: a demonstration how iterative leakage can pose a problem for federated power iteration; a further reduction in transmission cost, and increase in privacy, due to the use of randomized PCA; a data dependent speedup due to the use of approximate PCA. The remainder of this paper is organized as follows: In  \Cref{sec:prelim}, we introduce concepts and notations that are used throughout the paper. In \Cref{sec:related}, we discuss related work. In \Cref{sec:algo}, we describe the proposed algorithms. We then describe how to extract the covariance matrix from the updates in \Cref{sec:leakage}. In \Cref{sec:eval}, we report the results of the experiments. \Cref{sec:conc} concludes the paper.
 

%% file: preliminaries.tex

\section{Preliminaries}\label{sec:prelim}

\subsection{Federated learning and employed data model}

Typically, a star-like client-aggregator architecture is used in biomedical federated solutions \citep{Steed2010, splink2020}, with the data holders acting as clients. The data sets at the client sites will be called \emph{local data sets} and the parameters or models learned using this data will be called \emph{local parameters} or \emph{local models}, while the final aggregated model will be called \emph{pooled model}. The optimal result of the pooled model is achieved when it equals the result of the conventional model calculated on all data, which we call the \emph{global model}.

In federated settings, the data can be distributed in several ways. Either the clients observe a full set of variables for a subset of the samples (horizontal partitioning) or they have a partial set of variables for all samples (vertical partitioning) \citep{B2017,Wu2018}. In this paper, we assume that we are given a global data matrix $\A\in\mathbb{R}^{m\times n}$, where $m$ is the number of features (genetic variants, in the context of GWAS) and $n$ is the overall number of samples. The data is split across $S$ local sites as $\A=[\A^1\dots\A^s\dots\A^S]$, where $\A^s\in\mathbb{R}^{m\times n^s}$ and $n^s$ denotes the number of samples available at site $s$. From a semantic point of view, the partitioning is hence horizontal, since the samples are distributed over the local sites. However, from a technical point of view, the partitioning is vertical, since the samples correspond to the columns of \A. The reason for this rather unintuitive setup is that, when using PCA for GWAS, samples are treated as features, as detailed in the following paragraphs. 

\subsection{Principal component analysis and singular value decomposition}

Given a data matrix $\A \in \mathbb{R}^{m \times n}$, the PCA is the decomposition of the covariance matrix $\M=\A^\top \A\in \mathbb{R}^{n \times n}$ into $\M = \V\boldsymbol{\Sigma} \V^ \top$. $\boldsymbol{\Sigma}\in\mathbb{R}^{n\times n}$ is a diagonal matrix containing the eigenvalues $(\sigma_{i})_{i=1}^n$ of \M in non-increasing order, and $\V\in\mathbb{R}^{n \times n}$ is the corresponding matrix of eigenvectors \citep{Jolliffe2002}. Singular value decomposition (SVD) is closely related to PCA and an extension of PCA to non-square matrices. Many of the PCA algorithms actually call SVD to do the actual computation, because it is more efficient. Given a data matrix $\A \in \mathbb{R}^{m \times n}$, the SVD is its decomposition into $\A = \U\boldsymbol{\Sigma} \V^ \top$. The matrices $\U$ and $\V$ are the left and right singular vector matrices.  Usually, one is only interested in the top $k$ eigenvalues and corresponding eigenvectors. Since $k$ is arbitrary but fixed throughout this paper, we let $\G\in\mathbb{R}^{n\times k}$ and $\H\in\mathbb{R}^{m\times k}$ denote these first $k$ eigenvectors (\ie, \G corresponds to the first $k$ columns of $\V$).  $\G$ is typically used to obtain a low-dimensional representation $\A\mapsto\A\G\in\mathbb{R}^{m\times k}$ of the data matrix \A, which can then be used for downstream data analysis tasks. This, however, is not the way PCA is used in GWAS, as we will explain next.
\subsection{Genome-wide association studies}

The genome stores hereditary information that controls the phenotype of an individual in interplay with the environment. The genetic information is stored in the DNA encoded as a sequence of bases (A, T, C, G), the positions are called loci.  If we observe two or more possible bases at a specific locus in a population, we call this locus a \emph{single nucleotide polymorphism} (SNP). The predominant base in a population is called the \emph{major allele}; bases at lower frequency are called \emph{minor alleles} \citep{Tam2019}.

Genome wide association studies seek to identify SNPs that are linked to a specific phenotype \citep{Visscher2017, Tam2019}. Phenotypes of interest can for example be the presence or absence of diseases, or quantitative traits such as height or body mass index. The SNPs for a large cohort of individuals are tested for association with the trait of interest. Typically, simple models such as linear or logistic regression are used for this \citep{Visscher2017, splink2020}. The input to a GWAS is an $n$-dimensional phenotype vector $\y$, a matrix of SNPs $\A\in\mathbb{R}^{m\times n}$, and confounding factors such as age or sex, given as column vectors $\x_r\in\mathbb{R}^{n}$. The SNPs are encoded as categorical values between $0$ and $2$, representing the number of minor alleles observed in the individual at the respective position. Each SNP $l\in[m]$ is tested in an individual association test
\begin{equation}
\y \sim \beta_0 + \beta_1\cdot\A_{l,\bullet}^\top  + \sum_{r=1}^{R}\beta_{r+1}\cdot\x_r + \epsilon\text{,}
\end{equation}
where $\A_{l,\bullet}$ denotes the $l$\textsuperscript{th} row of \A.

\subsection{Principal component analysis for genome-wide association studies}

Confounding factors such as ancestry and population substructure can alter the outcome of an association test and create false hits if not properly controlled for \citep{Tam2019}. PCA has emerged as a popular strategy to infer population substructure and a SMPC based protocol has been presented by \citep{Cho2018}. More precisely, the first $k$ (usually $k=10$) eigenvectors $\G=[\g_1\dots\g_k]\in\mathbb{R}^{n\times k}$ of the sample covariance matrix $\A^\top\A$ are included into the association test as covariates \citep{Galinsky2016, Price2006}:

\
\begin{equation}
\y \sim \beta_0 + \beta_1\cdot\A_{l,\bullet}^\top  + \sum_{r=1}^{R}\beta_{r+1}\cdot\x_r + \sum_{i=1}^k\beta_{i+R+1}\cdot\g_i + \epsilon
\end{equation}

In federated GWAS, each local site $s$ needs to have access only to the partial eigenvector matrix $\G^s$ corresponding to the locally available samples. Consequently, computing the complete eigenvector matrix \G at the aggregator and/or sharing $\G^s$ with other local sites $s^\prime$ should be avoided to reduce the possibility of information leakage. This is especially important because \cite{sigir2021} have shown that, if \G is available at the aggregator in a federated GWAS pipeline, the aggregator can in principle reconstruct the raw GWAS data $\A_{l,\bullet}$ for SNP $l$. Federated PCA algorithms that are suitable for GWAS should hence have to respect the following constraint:
\begin{constraint}
In a GWAS-suitable federated PCA algorithm, the aggregator does not have access to the complete eigenvector matrix \G and each site $s$ has access only to its share $\G^s$ of \G.\label{constr:eig}
\end{constraint}

The PCA in GWAS is usually performed on only a subsample of the SNPs, but there seems to be no consensus as to how many SNPs should be used. Some PCA-based stratification methods rely on a small set of ancestry informative markers \citep{Li2016}, while others employ over \num{100000} SNPs \citep{Gauch2019}.

\begin{figure}[htb]
    \centering
    \input{figures/gwas_pca}
    \caption{Regular PCA for dimensionality reduction (top); GWAS PCA for sample stratification (bottom); and SVD (middle).}
    \label{fig:gwas-challenge}
\end{figure}
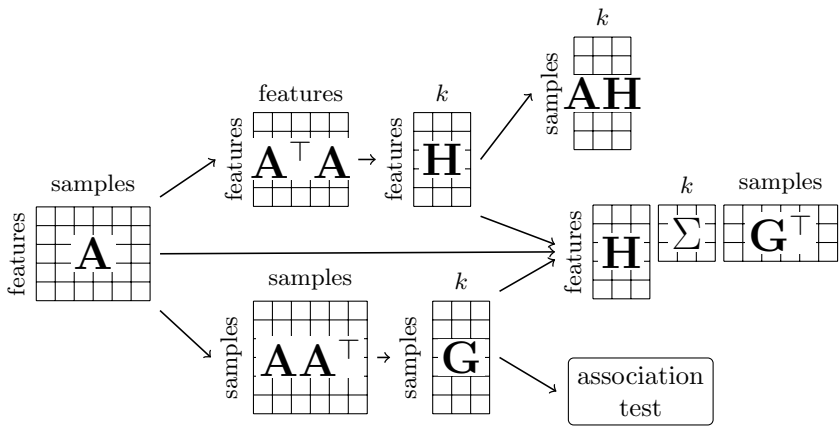

Note that PCA for GWAS is conceptually different from \enquote{regular} PCA for feature reduction (\cf \Cref{fig:gwas-challenge}). For feature reduction, we would decompose the $m \times m$ SNP by SNP covariance matrix and compute a set of \enquote{meta-SNPs} for each sample. This is not what is required for GWAS. Instead, the $n \times n$ sample by sample covariance matrix $\A^\top\A$ is decomposed. In our federated setting where \A is vertically distributed across local sites $s\in[S]$, $\A^\top\A$ looks as follows (recall that, unlike in regular PCA, columns correspond to samples and rows to features):
\begin{equation}
\A^\top\A = 
\begin{pmatrix}
{\A^1}^\top\A^1 & {\A^1}^\top\A^2 & \cdots & {\A^1}^\top\A^S \\
{\A^2}^\top\A^1 & {\A^2}^\top\A^2 & \cdots & {\A^2}^\top\A^S \\
\vdots & \vdots & \ddots & \vdots \\
{\A^S}^\top\A^1 & {\A^S}^\top\A^2 & \cdots & {\A^S}^\top\A^S
\end{pmatrix}
\end{equation}

It is clear that $\A^\top\A$ cannot be computed directly without sharing patient level data. Moreover, with a growing number of samples, this matrix can become very large and computing it becomes infeasible. For instance, the UK Biobank\,---\,a large cohort frequently used for GWAS\,---\,contains more than 4 million SNPs for more than \num{500000} individuals. Following directly from the definition of PCA, an exact computation of the covariance matrix would furthermore violate \Cref{constr:eig}. These considerations lead to the second constraint for federated PCA algorithms suitable for GWAS: 

\begin{constraint}
A GWAS-suitable federated PCA algorithm must work on vertically partitioned data and does not rely on computing or approximating the covariance matrix.\label{constr:cov}
\end{constraint}

\subsection{Gram-Schmidt orthonormalization}
The Gram-Schmidt algorithm transforms a set of linearly independent vectors into a set of mutually orthogonal vectors. Given a matrix $\V=[\v_1\dots\v_k]\in\mathbb{R}^{r\times k}$ of $k$ linearly independent column vectors, a matrix  $\U=[\u_1\dots\u_k]\in\mathbb{R}^{r\times k}$ of orthogonal column vectors with the same span can be computed as
\begin{equation}
\u_i = \begin{cases}
\v_i & \text{if }i=1\\
\v_i - \sum_{j=1}^{i-1}  r_{i,j} \cdot \u_j & \text{if }i\in[k]\setminus\{1\}
\end{cases}\text{,}
\label{eq:qr-proj}
\end{equation}
where $r_{i,j}=\u_j^\top\v_i/n_j$ with $n_j=\u_j^\top\u_j$.

The vectors can then be scaled to unit Euclidean norm as $\u_i\mapsto(1/\sqrt{n_i})\cdot\u_i$ to achieve a set of orthonormal vectors. In the context of PCA, this can be used to ensure orthonormality of the candidate eigenvectors in iterative procedures, which otherwise suffer from numerical instability in practice \citep{Guo2012}. 

\subsection{Notations}

\Cref{tab:notation} provides an overview of notations which are used throughout the paper.

\begin{table}[ht]
\centering
\caption{Notation table.}\label{tab:notation}
\begin{tabular}{ll}
\toprule
Syntax & Semantics \\
\midrule
$[N]\subset\mathbb{N}$ & index set $[N]=\{i\in\mathbb{N}\mid1\leq i\leq N\}$ \\
$S\in\mathbb{N}$ & number of sites \\
$m\in\mathbb{N}$ & number of features (\ie SNPs) \\
$n\in\mathbb{N}$ & total number of samples \\
$n^s\in\mathbb{N}$ & number of samples at site $s\in[S]$ \\
$k\in\mathbb{N}$ & number of eigenvectors \\
$\A\in\mathbb{R}^{m\times n}$ & complete data matrix \\
$\A^s\in\mathbb{R}^{m\times n^s}$ & subset of data available at site $s\in[S]$\\
$\G_i\in\mathbb{R}^{n\times k}$ & right singular matrix of $\A$ at iteration $i$ \\
$\G\in\mathbb{R}^{n\times k}$ & right singular matrix of $\A$\\
$\G_i^s\in\mathbb{R}^{n^s\times k}$ & partial right singular matrix of $\A$ at iteration $i$\\
$\G^s\in\mathbb{R}^{n^s\times k}$ & converged partial right singular matrix $\A$. \\
$\H_i\in\mathbb{R}^{m\times k}$ & left singular matrix of $\A$ at iteration $i$\\
$\H\in\mathbb{R}^{m\times k}$ & converged left singular matrix of $\A$ \\
$\V\in\mathbb{R}^{r\times k}$ & a generic column vector matrix\\
$\U\in\mathbb{R}^{r\times k}$ & an orthonormal matrix with $\thespan(\U)=\thespan(\V)$ \\
$\M\in\mathbb{R}^{m\times m}$ & exact covariance matrix\\
$\hat{\A},\hat{\M}, \hat{\H}$, $\hat{\G}$ & approximations of $\A$, $\M$, $\H$ and $\G$ \\
\bottomrule
\end{tabular}
\end{table}

%% file: figures/gwas_pca.tex
\begin{tikzpicture}[node distance=.6cm]
       
         \node (A-classic) {\tikz\draw[step=0.25] (0,0) grid (1.5,-1.25);};
       \node[fill=white,font=\Large,inner sep=1pt] (A-ckassic-label) at (A-classic) {\A};
       \node[rotate=90,anchor=south,font=\small,yshift=-.1cm] at (A-classic.west) {features};
       \node[anchor=south,font=\small,yshift=-.1cm] at (A-classic.north) {samples};

       \node[right = of A-classic, anchor=south west, xshift=0.5cm, yshift=0.5cm] (Sigma-classic) {\tikz\draw[step=0.25] (0,0) grid (1.25,-1.25);};
       \node[fill=white,font=\Large,inner sep=1pt] (Sigma-classic-label) at (Sigma-classic) {$\A^\top\A$};
       \node[rotate=90,anchor=south,font=\small,yshift=-.1cm] (Sigma-classic-y) at (Sigma-classic.west) {features};
       \node[anchor=south,font=\small,yshift=-.1cm] at (Sigma-classic.north) {features};
       
       \node[right = of Sigma-classic] (G-classic) {\tikz\draw[step=0.25] (0,0) grid (0.75,-1.25);};
       \node[fill=white,font=\Large,inner sep=1pt] (G-classic-label) at (G-classic) {\H};
       \node[rotate=90,anchor=south,font=\small,yshift=-.1cm] (G-classic-y) at (G-classic.west) {features};
       \node[anchor=south,font=\small,yshift=-.1cm] at (G-classic.north) {$k$};
       
       \node[right = of G-classic, anchor=south west, xshift=0.5cm] (AG-classic) {\tikz\draw[step=0.25] (0,0) grid (0.75,-1.5);};
       \node[fill=white,font=\Large,inner sep=1pt] (AG-classic-label) at (AG-classic) {$\A\H$};
       \node[rotate=90,anchor=south,font=\small,yshift=-.1cm] (AG-classic-y) at (AG-classic.west) {samples};
       \node[anchor=south,font=\small,yshift=-.1cm] at (AG-classic.north) {$k$};

       \node[right = of A-classic, anchor=north west, xshift=0.5cm, yshift=-0.5cm] (Sigma-gwas) {\tikz\draw[step=0.25] (0,0) grid (1.5,-1.5);};
       \node[fill=white,font=\Large,inner sep=1pt] (Sigma-gwas-label) at (Sigma-gwas) {$\A\A^\top$};
       \node[rotate=90,anchor=south,font=\small,yshift=-.1cm] (Sigma-gwas-y) at (Sigma-gwas.west) {samples};
       \node[anchor=south,font=\small,yshift=-.1cm] at (Sigma-gwas.north) {samples};
       
       \node[right = of Sigma-gwas] (G-gwas) {\tikz\draw[step=0.25] (0,0) grid (0.75,-1.5);};
       \node[fill=white,font=\Large,inner sep=1pt] (G-gwas-label) at (G-gwas) {\G};
       \node[rotate=90,anchor=south,font=\small,yshift=-.1cm] (G-gwas-y) at (G-gwas.west) {samples};
       \node[anchor=south,font=\small,yshift=-.1cm] at (G-gwas.north) {$k$};

       \node[right = of G-gwas, align=center, rounded corners=2pt,draw=black, anchor=north west, xshift=0.3cm] (test) {association\\test};
       
       \node[right = of G-gwas, anchor=south west, yshift=0.65cm, xshift=0.5cm] (G-svd) {\tikz\draw[step=0.25] (0,0) grid (0.75,-1.25);};
       \node[fill=white,font=\Large,inner sep=1pt] (G-svd-label) at (G-svd) {\H};
      \node[rotate=90,anchor=south,font=\small,yshift=-.1cm] (G-svd-y) at (G-svd.west) {features};
%
		 \node[right = of G-svd, xshift=-0.75cm, yshift=0.25cm] (sigma-svd) {\tikz\draw[step=0.25] (0,0) grid (-0.75,0.75);};
		\node[fill=white,font=\Large,inner sep=1pt] (sigma-svd-label) at (sigma-svd) {$\Sigma$};
		\node[anchor=south,font=\small,yshift=-.1cm] at (sigma-svd.north) {$k$};

       \node[right = of sigma-svd, anchor=west, xshift=-0.75cm] (H-svd) {\tikz\draw[step=0.25] (0,0) grid (-1.5,0.75);};
       \node[fill=white,font=\Large,inner sep=1pt] (H-svd-label) at (H-svd) {$\G^\top$};
       \node[anchor=south,font=\small,yshift=-.1cm] at (H-svd.north) {samples};

       \draw[->, semithick] (A-classic.north east) -- (Sigma-classic-y.north);
       \draw[->, semithick] (Sigma-classic.east) -- (G-classic-y.north);
       \draw[->, semithick] (G-classic.east) -- (AG-classic-y.north);
       
       \draw[->, semithick] (A-classic.south east) -- (Sigma-gwas-y.north);
       \draw[->, semithick] (Sigma-gwas.east) -- (G-gwas-y.north);
       \draw[->, semithick] (G-gwas.east) -- ($(test.west)+(-.15,0)$);
       
        \draw[->, semithick] (A-classic.east) -- (G-svd-y.north);
        \draw[->, semithick] (G-gwas.north east) -- ($(G-svd-y.north)+(0,-0.1)$);
        \draw[->, semithick] (G-classic.south east) -- ($(G-svd-y.north)+(0, +0.1)$);
       
  \end{tikzpicture}%

%% file: related.tex

\section{Related work}\label{sec:related}

\subsection{Centralized, iterative, covariance-free principal component analysis}

While classical PCA algorithms rely on computing the covariance matrix $\A^\top\A$ \citep{Jolliffe2002}, there are several covariance-free approaches to iteratively approximate the top $k$ eigenvalues and eigenvectors \citep{Saad2011}. \Cref{alg:subspace-centralised} summarizes the centralized, iterative, covariance-free PCA algorithm suggested by \cite{Halko_2010}, which will serve as point of departure for our federated approach. First, an initial eigenvector matrix is sampled randomly and orthonormalized (\crefrange{cent:init-start}{cent:init-end}). In every iteration $i$, improved candidate eigenvectors $\G_i$ of $\A^\top\A$ are computed (\crefrange{cent:while-start}{cent:while-end}). Once a suitable termination criterion is met (e.g., convergence, maximal number of iterations, time limit, etc.), the last candidate eigenvectors are returned (\cref{cent:ret}).

 \begin{algorithm}[ht]
	\caption{Vertical Subspace Iteration \cite{Halko_2010}.}
	\label{alg:subspace-centralised}
	\KwInput{Data matrix $\A\in\mathbb{R}^{m \times n}$, number of eigenvectors $k$.}
	\KwOutput{Singular matrices $\G \in \mathbb{R}^{n \times k}$ and $\H \in \mathbb{R}^{m \times k}$ of $\A$.}
	Generate $\G_0 \in \mathbb{R}^{n \times k}$ randomly\label{cent:init-start}\;
	$\G_0\gets  \orthonormalize(\G_0)$\label{cent:init-end}\;
	$i\gets 1$\;
	\While{termination criterion not met}{
		$\H_i = \A\G_{i-1}$\label{cent:while-start}\;
		$\H_i = \orthonormalize(\H_i)$\; 
		$\G_i = \A^\top\H_i$\; \label{cent:G-update}
		$\G_i \gets  \orthonormalize(\G_i)$\label{cent:while-end}\;
		$i\gets i+1$\;
	}
	\Return $\G_i$, $\H_i$\label{cent:ret}\;
\end{algorithm}

To update the candidate eigenvector matrices $\G_i=\A^\top\H_i=\A^\top\A\G_{i-1}\in\mathbb{R}^{n\times k}$ of $\A^\top\A$, the algorithm also computes candidate eigenvector matrices $\H_i=\A\G_{i-1} =  \A\A^\top\H_{i-1}\in\mathbb{R}^{m\times k}$ of $\A\A^\top$. Since, in the context of GWAS, $\A\A^\top$ corresponds to the \enquote{classical} feature by feature covariance matrix, and $\A^\top\A$ to the sample covariance matrix, the algorithm computes left and right singular vectors at the same time. This means, the present algorithm is actually an SVD algorithm. In this article, we will sometimes refer to the left singular vector as the feature eigenvector and the right singular vector as the sample eigenvector.

\subsection{Federated principal component analysis for vertically partitioned data}

Only few algorithms are designed to perform federated computation of PCA on vertically partitioned siloed data sets \citep{Guo2012,Kargupta2001,Qi2003,Wu2018}. However, none of them is suitable for the GWAS use-case considered in this paper: The algorithms reviewed by \cite{Wu2018} are specialised for distributed sensor networks and use gossip protocols and peer-to-peer communication. Therefore, they are not suited for the intended FL architecture in the medical setting. The algorithms presented by \cite{Kargupta2001} and \cite{Qi2003} rely on estimating a proxy covariance matrix and consequently do not meet \Cref{constr:cov} introduced above. Unlike these approaches, the algorithm proposed by \cite{Guo2012} is covariance-free and suitable for the intended star-like architecture. However, it broadcasts the eigenvectors to all sites in violation of \Cref{constr:eig}.

\subsection{Federated matrix orthonormalization}

Matrix orthonormalization is a frequently used technique in many applications, including the solution of linear systems of equations and singular value decomposition. There are three main approaches: Householder reflection, Givens rotation, and the Gram-Schmidt algorithm. In distributed memory systems and grid architectures, tiled Householder reflection is a popular approach \citep{HadriQR, HommenQR}. However, those algorithms are often highly specialized to the compute system and rely on shared disk storage. For distributed sensor networks, Gram-Schmidt procedures relying on push-sum have been proposed \citep{Sluciak2016,Strakov2012}. However, these methods require peer-to-peer communication and are hence unsuitable for the intended star-like architecture. Consequently, no federated orthonormalization algorithm suitable for our setup is available. In \Cref{sec:algo:fed-gs}, we present our own version of a federated orthonormalization algorithm fulfilling all constraints and subsequently utilize it as a subroutine in our federated PCA algorithm.

\subsection{Federated principal component analysis for horizontally partitioned data}

Previously, federated PCA algorithms have been described for horizontal and vertical data partitioning. In the remainder of this article, we establish an algorithm which is capable of both, which allows us to borrow ideas from previously described algorithms for horizontally federated PCA. There are \enquote{single-round} approaches, where the eigenvectors are computed locally and sent to the aggregator \citep{Balcan_2014}. At the aggregator, a global subspace is approximated from the local eigenspaces. The higher the number of transmitted intermediate dimensions, the better the global subspace approximation. In these algorithms, the solution quality hence depends on the number of transmitted dimensions. This algorithm is a more memory efficient version of the naive algorithm [\eg \citep{Liu2020}], where the entire covariance matrix is processed by the aggregator. Since only the top $k$ left singular values are transmitted, this algorithm fulfills \cref{constr:cov}. 
Furthermore, iterative schemes have been proposed, where locally computed eigenvectors are sent to the aggregator, which performs an aggregation step and sends the obtained candidate subspace back to the clients \citep{Balcan2016,Chen2020, Imtiaz2018}. The candidate subspace is then refined iteratively. Furthermore, there are several schemes for specific applications such as streaming \citep{SanchezFernandez2015,Grammenos2020}. These approaches assume that an approximation of the entire eigenvectors is possible at the clients, or that the global covariance matrix can be approximated. As we have discussed above, these assumptions do not hold in the intended GWAS use case.

\subsection{Randomized principal component analysis}
In the context of GWAS, a randomized PCA algorithm \citep{Halko_2010, Galinsky2016} is popular as it speeds-up the computation compared to traditional algorithms. Here, we briefly present the version implemented by \cite{Galinsky2016}. The algorithm starts with $I^\prime$ iterations of subspace iteration on the full-dimensional data matrix, resulting in feature eigenvetor matrices $H_i$ for all iterations $i\in\{1,\ldots,I^\prime\}$. Next, the data is projected on the concatenation of all $\H_i$, forming approximate principal components which approximate the data matrices. Then, subspace iteration is performed on these proxy data matrices. In practice, $I^\prime=10$ iterations are sufficient. This reduces the dimensionality of the data from $m$ to $k\cdot I^\prime$. In \Cref{sec:fed-approx} we will present a fully federated version of this algorithm in detail. Note that this is a randomized approach, because subspace iteration on the full-dimensional data is initialized randomly. Since it is interrupted before convergence after $I^\prime$ iterations, the feature eigenvetor matrices $H_i$ inherit this randomness.

%% file: algorithm.tex

\section{Algorithms}\label{sec:algo}

In this section, we present a federated SVD algorithm, which is designed for a star-like architecture, meets the requirements of \Cref{constr:eig} and \Cref{constr:cov}, and is hence suitable for federated GWAS. Our base algorithm comes in two variants\,---\,with and without orthonormalization of the candidate right singular vectors of \A. In \Cref{sec:algo:fed-pca}, we describe our algorithm and prove that the version with orthonormalization is equivalent to centralized vertical subspace iteration algorithm \citep{Halko_2010}, which we have summarized in \Cref{alg:subspace-centralised} above. In \Cref{sec:algo:fed-gs}, we present a federated Gram-Schmidt algorithm, which can be used as a subroutine in our federated SVD algorithm to ensure that right singular vectors of \A remain at the local sites. Again, we prove that our federated Gram-Schmidt algorithm is equivalent to the centralized counterpart. We then show how approximate horizontal PCA can be used to compute approximate principal components for immediate use or as initialization for subspace iteration in \Cref{sec:fed-approx}. In \Cref{sec:algo:randomized}, we present randomized federated subspace iteration as a means to reduce the transmission cost in federated SVD. In addition to decreasing the communication cost, the use of the two latter strategies also prevents potential iterative leakage (detailed in \Cref{sec:leakage}). In \Cref{sec:algo:costs}, we analyze the network transmission costs of the proposed algorithms, and \Cref{sec:algo:summary} provides an overview of possible configurations of our federated SVD algorithm.

\subsection{Federated vertical subspace iteration}\label{sec:algo:fed-pca}
\Cref{alg:subspace-federated} describes our federated vertical subspace iteration algorithm: Initially, the first partial candidate right singular matrices $\G_0^s$ of $\A$ are generated randomly and orthonormalized (\crefrange{fed-sub:init-start}{fed-sub:init-end}). Inside the main loop, the left singular vectors are updated at the clients, summed up element-wise and orthonormalized at the aggregator, and then sent back to the clients (\crefrange{fed-sub:H-start}{fed-sub:H-end}). Next, the clients update the partial right singular vectors (\cref{fed-sub:update-G}). In the version with orthonormalization, the candidate right singular vectors are now normalized by calling the federated Gram-Schmidt orthonormalization (\cref{fed-sub:fed-ortho}) algorithm presented in \Cref{sec:algo:fed-gs} (\Cref{alg:qr-federated}). Note that this algorithm ensures that the partial singular vectors $\G_i^s$ remain at the local sites. Finally, the full left singular matrices $\H$ and the orthonormalized partial right singular matrices $\G_s$  are returned to the clients (\cref{fed-sub:return}). In practice, the federated orthonormalization of $\G_i$ (\cref{fed-sub:fed-ortho}) may be omitted to speed up computation. Note, however, that $\H_i$ is still orthonormalized in every iteration and that the final orthonormalization (\cref{fed-sub:fed-ortho-mandatory}) is required.

\newcommand\tikzmk[1]{%
	\tikz[remember picture,overlay]\node[inner sep=2pt] (#1) {};}
\newcommand\boxit[2][]{\tikz[remember picture,overlay]{\node[yshift=4pt, xshift=4pt, fill=#1,opacity=.3,fit={(A)($(B)+(#2\linewidth,.8\baselineskip)$)}] {};}\ignorespaces}

\setlength{\fboxsep}{1pt}
\colorlet{royal}{gray!70}

\begin{algorithm}[ht]
	\small
	\caption{Federated vertical subspace iteration. \textcolor{gray!90}{(Partial) client-side computations are marked in gray.}}
	\label{alg:subspace-federated}
	\KwInput{Partial data matrices $\A^s\in \mathbb{R}^{m \times n^s}$ at sites $s\in[S]$, number of eigenvectors $k$, number of iterations $I$ and/or convergence threshold $\epsilon$.}
	\KwOutput{Partial right singular matrices $\G^s \in \mathbb{R}^{n^s \times k}$ and full left singular matrices $\H \in \mathbb{R}^{m \times k}$ of $\A$ at sites $s\in[S]$.}
	\tikzmk{A}
	\If{use approximate initialisation}{
		 \tcp{Call subroutine \Cref{alg:approx-smpc} (\FEDAI).} 
 	$\G^s_0 \gets \initApprox()$ \label{fed-sub:init-approx}}
\Else{
	\tcp{Use random initialisation if no initial eigenvector is available (\FEDRI).}
	\lFor{$s\in[S]$}{generate $\G_0^s \in \mathbb{R}^{n^s \times k}$ \label{fed-sub:init-start} randomly}
	\tcp{Use approach described in \Cref{alg:qr-federated} (\FEDGS).}
	\lIf{use orthonormalization}{ $\fedOrthonormalize([\G_0^s])$\label{fed-sub:orthonormalize}\label{fed-sub:init-end}}}
	\tikzmk{B}
	\boxit[royal]{0.94}
	$i\gets 1$\;
	\tcp{Suggested criterion: $i\geq I$ or convergence as specified in \cref{eq:convergence-power}.}
	\While {termination criterion not met}{
		\tikzmk{A}
		\tcp{Update left singular matrix of $A$.}
		\lFor{$s\in[S]$}{$\H^s_{i} \gets \A^s\G^s_{i-1}$\label{fed-sub:H-start}}
		\tikzmk{B}
		\boxit[royal]{0.89}
		$\H_i \gets \sum_{s=1}^S\H_{i}^s$\label{fed-sub:aggregate-H}\;
		$\H_i \gets \orthonormalize(\H_i)$\label{fed-sub:H-ortho}\label{fed-sub:H-end}\;
		\tikzmk{A}
		\tcp{Update partial right singular matrices of $\A$.}
		\lFor{$s\in[S]$}{$\G_i^s \gets {\A^s}^\top \H_i$\label{fed-sub:update-G}}
		\tcp{Use approach described in \Cref{alg:qr-federated} (\FEDGS).}
		\lIf{use orthonormalization}{$\fedOrthonormalize([\G_i^s])$\label{fed-sub:fed-ortho}}
		\tikzmk{B}
		\boxit[royal]{0.89}
		$i\gets i+1$\;}
	\tikzmk{A}
	\For{$s\in[S]$}{$\G^s\gets \G_i^s$\;}
	$\G^s \gets \fedOrthonormalize([\G^s])$\label{fed-sub:fed-ortho-mandatory}\;
	\tikzmk{B}
	\boxit[royal]{0.94}
	\Return $\G^s, \H$\label{fed-sub:return}\; 
\end{algorithm}

Like the original centralized version described in \Cref{alg:subspace-centralised} above, our algorithm can be run with various termination criteria. In our implementation, we use the convergence criterion 
\begin{equation}
\diag(\H_i^\top\H_{i-1}) \geq \mathbf{1}_k - \epsilon
\label{eq:convergence-power}   
\end{equation}
using the angle as a global measure as suggested by \cite{Lei2016}, where $\mathbf{1}_k$ is the $k$-dimensional vector of ones and $\epsilon$ is a small positive number. With this criterion, the algorithm terminates once all right singular vectors of $\A$ are asymptotically collinear with respect to the eigenvectors of the previous iteration. Other convergence criteria could be used as drop-in replacements.

We now prove that the version of \Cref{alg:subspace-federated} with orthonormalization is equivalent to the centralized version described in \Cref{alg:subspace-centralised}. Thus, it inherits its convergence behavior from the centralized version. Details on the convergence behavior of centralized vertical subspace iteration can be found in the original publication by \cite{Halko_2010}.

\begin{proposition}
\label{prop:fed-pca}
If orthonormalization is used, centralized and federated vertical subspace iteration are equivalent.
\end{proposition}

\begin{proof}
Let $\G_i$ and $\H_i$ denote the eigenvector matrices maintained by the centralized algorithm described in \Cref{alg:subspace-centralised} at the end of the main while-loop, and $\G_i^s$ be the sub-matrix of $\G_i$ for the samples available at site $s$. Moreover, let $\tilde{\H}_i$, $\tilde{\G}_i$, $\tilde{\G}^s_i$, and $\tilde{\H}^s_i$ be the (partial) eigenvector matrices maintained by our federated \cref{alg:subspace-federated} at the end of the main while-loop. We will show by induction on the iterations $i$ that $\H_i=\tilde{\H}_i$ and $\G_i^s=\tilde{\G}_i^s$ for all $s\in[S]$ holds throughout the algorithm, if the same random seeds are used for initialization.

For $i=0$, we only have to show $\G_0^s=\tilde{\G}_0^s$. This directly follows from \cref{prop:fed-gs} and our assumption that the same random seeds are used for initialization. For the inductive step, note that, before orthonormalization in \cref{fed-sub:H-ortho}, we have $\tilde{\H}_i=\sum_{s=1}^S\tilde{\H}_i^s=\sum_{s=1}^S\A^s\tilde{\G}_{i-1}^s=\sum_{s=1}^S\A^s\G_{i-1}^s=\A\G_{i-1}=\H_i$, where the third equality follows from the inductive assumption. Because of \Cref{prop:fed-gs}, this identity continues to hold at the end of the main while-loop.

Similarly, after updating in \cref{fed-sub:update-G} but before orthonormalization, we have $\tilde{\G}^s_i={\A^s}^\top\tilde{\H}_i={\A^s}^\top\H_i=(\A^\top\H_i)^s=\G_i^s$, where the second equality follows the identity $\H_i=\tilde{\H}_i$ shown above and $(\A^\top\H_i)^s$ denotes the sub-matrix of $\A^\top\H_i$ for the samples available at site $s$. Again, \cref{prop:fed-gs} ensures that the identity continues to hold after orthonormalization.
\end{proof}

The omission of the orthonormalization of $\G_i$ (\cref{fed-sub:orthonormalize} and \cref{fed-sub:fed-ortho} in \Cref{alg:subspace-federated}) removes provable identity to  \cref{alg:subspace-centralised}. However, other formulations of centralized power iteration exist which directly operate on the covariance matrix \citep{Balcan2016}. In these schemes, instead of splitting the iteration into $\H_i$ update (\cref{cent:while-start}, \cref{alg:subspace-centralised}) and $\G_i$ update (\cref{cent:G-update}, \cref{alg:subspace-centralised}), the covariance matrix is computed and $\H_i$ is updated as $\H_i=\A\A^\top\H_{i-1}$ at every iteration. \Cref{prop:fed-pca} can be formulated and proven analogously for this version.

\subsection{Federated Gram-Schmidt algorithm}\label{sec:algo:fed-gs}
Here, we describe federated Gram-Schmidt orthonormalization for vertically partitioned column vectors. Previous federated PCA algorithms require the complete eigenvectors to be known at all sites for the orthonormalization procedure. The na\"ive way of orthonormalizing the eigenvector matrices in a federated fashion would be to send them to the aggregator which performs the aggregation and then sends the orthonormal matrices back to the clients. However, in this na\"ive scheme, the transmission cost scales with the number of variables (individuals in GWAS) and all eigenvectors are known to the aggregator.

To address these two problems, we suggest a federated Gram-Schmidt orthonormalization procedure, summarized in \Cref{alg:qr-federated}. The algorithm exploits the fact that the computations of the squared norms $n_i$ and of the residuals $r_{ij}$ can be decomposed into independent computations of summands $n_i^s$ and $r_{ij}^s$ computable at the local sites $s\in[S]$. The clients compute the local summands and send them to the aggregator, where the squared norm of the first orthogonal vector is computed and sent to the clients (\crefrange{gs-fed:init-start}{gs-fed:init-end}). Subsequently, the remaining $k-1$ vectors are orthogonalized. For the $i$\textsuperscript{th} vector $\v_i$, the algorithm computes the residuals $r_{ij}$ \wrt all already computed orthogonal vectors $\u_j$, using the fact that the corresponding squared norms $n_j$ are already available (\crefrange{gs-fed:res-start}{fed-gs:res-end}). The residuals are aggregated by the central server (\crefrange{gs-fed:aggregated-res-start}{gs-fed:aggregated-res-end}). Next, $\v_i$ is orthogonalized at the clients,  the local norms are computed (\crefrange{gs-fed:norm-start}{gs-fed:norm-end}), and the squared norm of the resulting orthogonal vector $\u_i$ is computed at the aggregator and sent back to the clients (\cref{gs-fed:aggregate-norm}). After orthogonalization, all orthogonal vectors are scaled to unit norm at the clients (\crefrange{gs-fed:scale-start}{gs-fed:scale-end}).

\begin{algorithm}[ht]
	\small
	\caption{Federated Gram-Schmidt. \textcolor{gray!90}{Client-side computations are marked in gray.}}
	\label{alg:qr-federated}
	\KwInput{Data matrices $\V_s$ at sites $s\in[S]$.}
	\KwOutput{Orthonormalized data matrices $\U_s$ at sites $s\in[S]$.}
	\tikzmk{A}
	\tcp{Compute squared norm of first orthogonal vector.}
	\For{$s\in[S]$\label{gs-fed:init-start}}
	{$\u^s_1 \gets \v^s_1$\; 
		$n_1^s \gets {\u_1^s}^\top{\u_1^s}$\;}
	\tikzmk{B}
	\boxit[royal]{0.94}
	$n_1 \gets \sum_{s=1}^S n_1^s$ \; \label{gs-fed:init-end}
	\tcp{Orthogonalize all subsequent vectors.}
	\For{$i \in [k]\setminus\{1\}$}{
		\tikzmk{A}
		\tcp{Compute client residuals for vector being orthogonalized.}
		\For{$s\in[S]$\label{gs-fed:res-start}}{\For{$j \in[i-1]$}{
				$r_{ij}^s \gets {\u_j^s}^\top \v_i^s / n_j $\; \label{fed-gs:res-end}
		}}
		\tikzmk{B}
		\boxit[royal]{0.89}
		\tcp{Compute global residuals for vector being orthogonalized.}
		\For{$j\in[i-1]$\label{gs-fed:aggregated-res-start}}{
			$r_{ij} \gets \sum_{s=1}^Sr_{ij}^s$\; \label{gs-fed:aggregated-res-end}
		}
		\tikzmk{A}
		\tcp{Orthogonalize the vector and compute squared norm.}
		\For{$s\in[S]$\label{gs-fed:norm-start}}{
			$\u_{i}^s \gets \v_{i}^s - \sum_{j =1}^{i-1} r_{ij} \cdot \u_{j}^s$\label{gs-client:orth}\;
			$n_i^s \gets {\u_i^s}^\top \u_i^s$\label{gs-fed:norm-end}
		}\tikzmk{B}
		\boxit[royal]{0.89}
		\tcp{Compute squared norm of orthogonalized vector.}
		$n_i \gets \sum_{s=1}^S n_i^s$ \; \label{gs-fed:aggregate-norm}
	}
	\tikzmk{A}
	\tcp{After orthogonalization, scale all $k$ vectors to unit norm.}
	\For{$s\in[S]$\label{gs-fed:scale-start}}{
		\lFor{$i \in [k]$}{$\u_i^s \gets \frac{1}{\sqrt{n_i}}\cdot\u_i^s$}
		$\U^s\gets[\u_1^s\dots \u^k_s]$\;
		 \label{gs-fed:scale-end}} \tikzmk{B}\boxit[royal]{0.94}
	 \Return $\U^s$\;
\end{algorithm}

\begin{proposition}
\label{prop:fed-gs}
Centralized and federated Gram-Schmidt orthonormalization are equivalent.

\end{proposition}
\begin{proof}
Let $\V=[\v_1\dots\v_k]$ be the matrix that should be orthonormalized, $\v_i^s$ be the restriction of the $i$\textsuperscript{th} columns vector to the samples available at side $s$, and $\u_i^s$ be the restriction of the $i$\textsuperscript{th} orthogonal vector computed by the centralized Gram-Schmidt algorithm before normalization to the samples available at site $s$. Moreover, let $n_i$ and $r_{i,j}$ be the centrally computed norms and residuals, and $\tilde{n}_i$, $\tilde{r}_{i,j}$, and $\tilde{\u}_i^s$ be the locally computed norms, residuals, and partial orthogonal vectors before normalization. We show by induction on $i$ that $n_i=\tilde{n}_i$, $r_{ij}=\tilde{r}_{ij}$, and $\u_i^s=\tilde{\u}_i^s$ holds for all $i\in[k]$ and all $j\in[i-1]$. This implies the proposition.

For $i=1$, we have $\u_1^s=\v_1^s=\tilde{\u}_1^s$ and $n_1=\u_1^\top\u_1=\sum_{s=1}^S{\u_1^s}^\top \u_1^s=\sum_{s=1}^S{\tilde{\u}^{s\top}_1}\tilde{\u}_1^s=\tilde{n}_1$. For the inductive step, note that $r_{ij}={\u_j}^\top\v_i/n_j=\sum_{s=1}^S{\u_j^s}^\top \v_i^s/n_j=\sum_{s=1}^S\tilde{\u}_j^{s\top}\v_i^s/\tilde{n}_j=\tilde{r}_{ij}$, where the third identity follows from the inductive assumption. Moreover, we have $\u_{i}^s=\v_{i}^s-\sum_{j=1}^{i-1}r_{ij}\cdot \u_{j}^s=\v_{i}^s-\sum_{j=1}^{i-1}\tilde{r}_{ij}\cdot\tilde{\u}_{j}^s=\tilde{\u}_{i}^s$, where the second identity follows from the inductive assumption and the identities $r_{ij}=\tilde{r}_{ij}$ established before. We hence obtain $n_i=\u_i^\top\u_i=\sum_{s=1}^S{\u_i^s}^\top \u_i^s=\sum_{s=1}^S{\tilde{\u}^{s\top}_i}\tilde{\u}_i^s=\tilde{n}_i$, which completes the proof.
\end{proof}


\subsection{Approximate initialization} \label{sec:fed-approx}
One major concern of iterative PCA is information leakage through the repeated transmission of updated eigenvectors. This is presented in more detail in \Cref{sec:leakage}, because knowledge of the subspace iteration algorithm is required to understand the attack. Briefly, the conclusion is that the number of iterations needs to be strictly limited. Therefore, we suggest to use federated approximate horizontal PCA as an initialization strategy to limit the number of iterations, and thereby prevent the possible leakage of the covariance matrix.

\cite{Balcan_2014} presented a memory efficient version of federated approximate PCA for horizontally partitioned data. We provide a minor modification which allows us to compute the sample eigenvectors. The algorithm can be used \enquote{as is} to compute the federated approximate vertical PCA by projecting the approximate left eigenvector to the data; or as an initialization strategy for federated subspace iteration. For the latter, instead of initializing $\G^s_0$ randomly (\cref{fed-sub:init-start}, \Cref{alg:subspace-federated}), $\G^s_0$ is computed using the approximate algorithm described here (\cref{fed-sub:init-approx}, \Cref{alg:subspace-federated}). 

\Cref{alg:approx-smpc} describes this approach. The algorithm proceeds as follows: At the clients, a local PCA is computed and the top $2k$ eigenvectors are shared with the aggregator with $c$ a constant multiplicative factor (\cref{approx:init}). At the aggregator, the local eigenvectors are stacked such that a new approximate covariance matrix $\hat{{\M}}$ with $\dim(\hat{{\M}})= c \cdot k \cdot S\times m$ is formed. $\hat{{\M}}$ is then decomposed using singular value decomposition leading to a new eigenvector estimate $\hat{{\H}}$ (\crefrange{approx:agg-start}{approx:agg-end}). At the clients, the feature eigenvector estimate $\hat{{\H}}$ can be projected onto the data to form an approximation of the sample eigenvector $\hat{\G}_s$. The vectors $\hat{\G}$ and $\hat{\H}$ represent an \enquote{educated guess} of the final singular vectors.

\begin{algorithm}[ht]

 \caption{Slightly modified federated horizontal SVD   \citep{Balcan_2014}. Referred to as \AIONLY in this article.}
 \label{alg:approx-smpc}

 \KwInput{Data matrices $\A_s\in\mathbb{R}^{m \times n}$ at sites $s \in [S]$, number of eigenvectors $k$, constant approximation factor $c$.}
 \KwOut{Approximate singular vector matrices $\hat{\G}_s \in \mathbb{R}^{n_s \times k}$ and $\hat{\H} \in \mathbb{R}^{m \times k}$ of $\A$.}
\tikzmk{A}
\For{$s\in[S]$}{
\tcp{Retrieve top $k \cdot c$ eigenvectors.}
 $\H_s,\boldsymbol{\Sigma}_s, \G_s \gets \SVD(\A_s, c \cdot k)$\; \label{approx:init}
}
\tikzmk{B}
\boxit[royal]{0.94}
\tcp{Aggregate local subspaces to obtain approximate covariance matrix $\hat{{\M}}$ with $\dim(\hat{{\M}})= c \cdot k \cdot S\times m$.}
 $\hat{{\M}} \gets \stackVert ([\H^\top_s])$ \label{approx:agg-start}\;
 \tcp{Use final dimensionality $k$}
 $\hat{{\H}} \gets \SVD(\hat{{\M}}, k)$\; \label{approx:agg-end}
\tikzmk{A}
\For{$s\in[S]$}{$\hat{{\G}}_s \gets \A_s^{\top} \hat{{\H}}$\; \label{approx:project}}
\tikzmk{B}
\boxit[royal]{0.94}
\tcp{Return approximate singular vector matrices of $\A$}
\Return $\hat{{\G}}_s,\hat{{\H}}$\label{approx:return}

 \end{algorithm}

\subsection{Federated randomized principal component analysis}\label{sec:algo:randomized}
Another mitigation strategy for the aforementioned information leakage is the use of randomized SVD. In randomized SVD, a reduced representation of the data is computed and subspace iteration is applied on this reduced data matrix instead of the full data. By using the proxy data, only \enquote{reduced} eigenvectors become available at the aggregator which makes the attack in \Cref{sec:leakage} impossible given not too many initial iteration $I^\prime$ have been executed. Notably, $I^\prime$ needs to be restricted depending on the number of features in the original data.
Here, we describe how to modify randomized SVD, such that it can be run in a federated environment, without sharing the random projections of the data or the sample eigenvectors. 

We proceed according to \cite{Halko_2010} and \cite{Galinsky2016}. First, $I^\prime$ iterations of federated vertical subspace iteration are run using the full data matrices $\A_s$. In order to do so, \Cref{alg:subspace-federated} is called as a subroutine. The intermediate matrices $\H_1,\ldots,\H_I^\prime$ are stored (\cref{rand:init-subs}) and concatenated to form $\P \in \mathbb{R}^{k\cdot I^\prime \times m}$ (\cref{rand:concat}). The data matrices $\A_s$ are then projected onto $\P$ to form proxy data matrices $\hat{{\A}}_s \in \mathbb{R}^{k\cdot I^\prime \times n}$ (\cref{rand:approx}). Finally, the covariance matrix of the proxy data matrix is computed as $\hat{{\A}}_s\hat{{\A}}_s^\top \in \mathbb{R}^{k\cdot I^\prime \times k\cdot I^\prime}$ at the clients. The clients send this covariance to the aggregator which aggregates the covariance matrices by element wise addition $M_{\hat{\A}} = \sum_{s} \hat{{\A}}_s\hat{{\A}}_s^\top$, computes the eigenvectors $\G_a$ and shares them with the clients. The eigenvectors $\G_{\hat{\A}}  \in \mathbb{R}^{k\cdot I^\prime \times k}$ do not reflect properties of the original data but they can be used to recompute $\G$ as $\G= \hat{{\A}}_s^\top \G_{\hat{\A}}  \in \mathbb{R}^{n \times k}$. $\G$ needs to be normalized using the federated orthonormalization subroutine (\cref{alg:qr-federated}). The subroutine returns the correct right singular vectors $\G$ but only proxy vectors for $\H$. Therefore, in the last step $\H$ can be reconstructed by projecting the data onto $\G$, aggregating and normalizing $\H_s$ at the aggregator and returning the final left singular vectors $\H$ to the clients (\crefrange{rand:H-start}{rand:H-end}). 

We would like to highlight two properties of this algorithm. Firstly, given that $m>I/k$, it is not possible to construct the covariance matrix using the initial eigenvector updates (see \Cref{sec:leakage}). Secondly, since the computation of the final right singular vectors utilizes the projected data matrices, the original covariance matrix is not disclosed exactly. However, as the final left eigenvectors are a common result of the analysis, the covariance matrix can still be approximated closely.

 \begin{algorithm}[ht]
	\caption{Federated randomized SVD (\FEDRANDRI)}
	\label{alg:approx-random}
	\KwInput{Data matrices $\A_s\in\mathbb{R}^{m \times n}$ at sites $s \in [S]$, number of eigenvectors $k$, number of intermediate iterations $I'$, number of total iterations $I$.}
	\KwOutput{Singular vector matrices $\H \in \mathbb{R}^{m \times k}$ and partial $\G_s \in \mathbb{R}^{n_s \times k}$ of $\A$.}
	\tcp{Run $I'$ iterations of \Cref{alg:subspace-federated} and store $H_i$.}
	 $[\H_1, ..., \H_{I^{'}} ]  \gets \fedSub([\A_s], k, I')$ \label{rand:init-subs}\; 
	 $\P \gets \stackVert([\H_1^\top, ..., \H_{I^{'}}^\top ])$ \;  \label{rand:concat}
	\tikzmk{A}
	\For{$s\in[S]$} {
	 $\hat{\A}_s \gets \P\A_s$\; \label{rand:approx} 
	 $\M_{\hat{\A}}^s = \hat{{\A}}_s\hat{{\A}}_s^\top$ \label{rand:cov}
	}
	\tikzmk{B}
	\boxit[royal]{0.94}
	 $\M_{\hat{\A}} = \sum_{s} \hat{{\A}}_s\hat{{\A}}_s^\top$\;
	 $\G_{\hat{\A}} \gets \SVD(\M_{\hat{\A}}, k)$\; 
	 
	 \tikzmk{A}
	 {\tcp{Reconstruct the partial right singular vector}
	 $\G_s \gets \hat{{\A}}_s^\top \G_{\hat{\A}} $\;
	 }
	 \tikzmk{B}
	\boxit[royal]{0.94}
	 \tcp{Reorthogonalize the partial right singular vector}
	  $\G_s  \gets \fedOrthonormalize([\hat{\G}_s], k, I)$\;\label{rand:full-subs}
	 \tikzmk{A}
	 \tcp{Compute the full left singular vectors of $A$}
	 \lFor{$s\in[S]$}{$\H_s \gets \A_s\G_s$\label{rand:H-start}}
	 \tikzmk{B}
	 \boxit[royal]{0.89}
	 $\H \gets \sum_{s=1}^S\H_s$\label{rand:aggregate-H}\;
	 $\H \gets \orthonormalize(\H)$\label{rand:H-end}\;
	\tcp{Return singular vector matrix of $\A$.}
	\Return $\G_s, \H$\label{rand:return}
\end{algorithm}

\subsection{Network transmission costs}\label{sec:algo:costs}
The main bottleneck in FL is the amount of data transmitted between the different sites and the number of network communications and the volume of transmitted date \citep{Kairouz2021}. The following \Cref{prop:costs} specifies these quantities for our federated PCA algorithm. Recall that $S$, $k$, $m$, $n$, and $c$ denote, respectively, the numbers of sites, eigenvectors, features, samples, and a constant multiplicative factor.

\begin{proposition}
	
\label{prop:costs}
Let $\mathcal{D}$ be the total amount of data transmitted by the federated SVD algorithm, $\mathcal{N}$ be the total number of network communications, and $I$ be the total number of iterations of the main while-loop. Let further $I'$ be the number of inital iteration for randomized SVD and $k'$ the intermediate dimensionality of the subspace for the approximate algorithm. Then the following statements hold:	

\begin{itemize}
\item If the $\G_i$ matrices are not orthonormalized, then $\mathcal{D}=\mathcal{O}(I\cdot S\cdot k\cdot m)$ and $\mathcal{N}=\mathcal{O}(I\cdot S)$.
\item If federated Gram-Schmidt orthonormalization is used, then $\mathcal{D}=\mathcal{O}(I\cdot (S\cdot k\cdot m + k^2))$ and $\mathcal{N}=\mathcal{O}(I\cdot S\cdot k)$.
\item If federated randomized subspace iteration is used, then  $\mathcal{O}((S\cdot (I'+1) \cdot k \cdot m)+(k \cdot I')^2)$ and $\mathcal{N}=\mathcal{O}((I'+ 2)\cdot S)$. 
\item Approximate initialization itself has a complexity of $\mathcal{D}=\mathcal{O}(S\cdot k \cdot c \cdot m)$ and $\mathcal{N}=\mathcal{O}(S)$, hence the other algorithms remain in the same complexity class if used in combination with approximate initialization.
\end{itemize}
\end{proposition}

\begin{proof}
	In each iteration $i$ of our federated vertical subspace iteration algorithm, the matrices $\H_i^s\in\mathbb{R}^{m\times k}$ have to be sent from the clients to the aggregator and the matrix $\H_i\in\mathbb{R}^{m\times k}$ has to be sent back to the clients. In iteration $i$, the amount of transmitted data and the number of communications due to $\H_i$ is hence $\mathcal{O}(S\cdot k\cdot m)$ and $\mathcal{O}(S)$, respectively. For orthonormalizing the eigenvector matrices $\G_i\in\mathbb{R}^{n\times k}$, we need to transmit a data volume of $\mathcal{O}(S\cdot k^2)$ and the number of communications increases to $\mathcal{O}(S\cdot k)$. By summing over the iterations $i$, this yields the statement of the proposition. In the randomized iteration the first $I'$ iterations have the same communication complexity that regular subspace iteration. Then the dimensionality of the matrix is reduced to $k\cdot I'\times n$ and the decomposition of $\M_{\hat{\A}}$ has transmission complexity $(k\cdot I')^2$. An additional communication of the final $\H$ is required which costs $S \cdot k \cdot m$. Thereby, the total complexity is $\mathcal{D}= \mathcal{O}(I'\cdot S \cdot k \cdot m + (I'\cdot k)^2 + S \cdot k \cdot m) = \mathcal{O}((S\cdot (I'+1) \cdot k \cdot m)+(k \cdot I')^2)$. Approximate initialization has a complexity of one round of subspace iteration, as $\H_i$ needs to be communicated once to the aggregator and back. The complexity classes hence remain the same.
\end{proof}

If our algorithms are used, the overall volume of transmitted data is hence independent of the number of samples $n$ and can be executed in a constant number of communication rounds. This is especially important in the intended GWAS setting, since the number of samples and features may ber very large \citep{Li2016, Londin2010}. Moreover, $k$ is small (typically, $k=10$ is used for GWAS PCA), which implies that the additional factor $k$ in the complexities of $\mathcal{D}$ and $\mathcal{N}$ can be neglected. Therefore, using the suggested scheme is preferable over sending the eigenvectors to the aggregator for orthonormalization both in terms of privacy and expected transmission cost. (In practice, it is advisable to perform the orthonormalization only at the end). \cite{Guo2012}'s algorithm has a complexity of $\mathcal{D}=\mathcal{O}(I\cdot S\cdot k)$ and $\mathcal{N}=  \mathcal{O}(I\cdot S)$ per eigenvector. The use of randomized SVD additionally partially removes the dependency of the algorithm from the number of SNPs/features which can be quite large in practice. (The worst case complexity class does not change due to the first iterations). Additionally, only a few iterations of the true feature eigenvectors are transmitted. Therefore, randomized SVD is preferable in terms of privacy and transmission cost.

\subsection{Summary}
\label{sec:algo:summary}
To conclude this section, we provide a brief summary of the main points and introduce a naming scheme for the configurations evaluated in \Cref{sec:eval}. We presented federated vertical subspace iteration with random (\FEDRI) initialization. To avoid the sharing of the sample eigenvector matrix, we introduced federated Gram-Schmidt orthonormalization (\FEDGS) which can be run at every iteration, but should be run only at the end. In order to speed up the computation in terms of communication rounds, we suggest to use a modified version of the approximate algorithm (\AIONLY) by \cite{Balcan_2014} as an initialization strategy for federated subspace iteration (\FEDAI). To reduce the transmitted data volume and the sharing of the feature eigenvectors, we suggest to use federated randomized subspace iteration (\FEDRANDRI).  \GUO is the reference algorithm. We summarize the asymptotic communication costs in \Cref{tab:cost-overview}.

 \begin{table}[ht]
	\centering
	\caption{Algorithm overview and complexity.}
	\label{tab:cost-overview}
	\begin{tabular}{llll}
		\toprule
		Algorithm(s) & Name & $\mathcal{D}$ & $\mathcal{N}$ \\
		\midrule
		\Cref{alg:subspace-federated} &\FEDRI & $\mathcal{O}(I\cdot S\cdot k\cdot m)$ &  $\mathcal{O}(I\cdot S)$.\\
		\Cref{alg:subspace-federated}+\ref{alg:approx-smpc} &\FEDAI  & $\mathcal{O}(I\cdot S\cdot k\cdot m)$ &  $\mathcal{O}(I\cdot S)$.\\
		\Cref{alg:subspace-federated}+\ref{alg:qr-federated} &\FEDRI/\FEDGS & $\mathcal{O}(I\cdot (S\cdot k\cdot m + k^2))$ & $\mathcal{O}(I\cdot S\cdot k)$. \\
		\Cref{alg:approx-random} &\FEDRANDRI & $\mathcal{O}((S\cdot (I'+1) \cdot k \cdot m)+(k \cdot I')^2)$ & $\mathcal{O}((I'+2)\cdot S)$.\\
		\Cref{alg:approx-smpc} &\AIONLY & $\mathcal{O}(S\cdot k\cdot m)$ &  $\mathcal{O}(S)$.\\
		\cite{Guo2012} &\GUO & $\mathcal{O}(I\cdot S\cdot k\cdot m)$ &  $\mathcal{O}(I\cdot S)$.\\
		\bottomrule
	\end{tabular}
\end{table}

%% file: leakage.tex
\section{Iterative leakage at the aggregator}
\label{sec:leakage}
In this section, we describe how the iterative process discloses the covariance matrix when using sufficiently many iterations. We first introduce the problem (\Cref{sec:leakage:problem}) and then discuss how it can be addressed with the algorithms introduced \Cref{sec:fed-approx,sec:algo:randomized} above (\Cref{sec:leakage:solution}). Practical results are illustrated in \Cref{it-leak:numerical} below, using a simulation study.

\subsection{Iterative leakage of the covariance matrix}\label{sec:leakage:problem}

Iterative leakage at the aggregator might disclose the entire covariance matrix during the execution of the algorithm, as many updates of the variables become available. \Cref{fig:lin-eq} visualizes the update process in power iteration, and the information used to reconstruct a single row of the covariance matrix at one iteration. Notably, the aggregated vector $\H_i$ becomes known in clear text at the aggregator in every iteration. The aggregator can store the sequence of vectors $\H_i$. In the following we will show, how it is possible to construct a system of linear equations which will leak the covariance matrix. For the sake of this description, we will assume the eigenvector $\H_i$ is updated as $\H_i= \K\H_{i-1}$, where $\K=\D^\top\D$ is the feature-by-feature covariance matrix of the federated data matrix $\D$, which are both unknown to the aggregator. This is equivalent to the two-step update from the aggregator's perspective, but improves the readability.

\begin{figure}[ht]
	\centering
	\begin{tikzpicture}[node distance=.6cm]
	\node (A-classic) {\tikz\draw[step=0.25] (0,0) grid (1.5,-1.5);};
	\node[fill=white,font=\large,inner sep=1pt] (A-classic-label) at (A-classic) {\K};
	
	\node[anchor=south,font=\small,yshift=-15pt, xshift=-20pt] (k-label) at (A-classic.north west) {$K_{l,\bullet}$};
	\node[anchor=south,font=\small,yshift=-11pt] (k-label-anchor) at (A-classic.north west) {};
	
	\node[right = of A-classic, xshift=-15pt] (G-classic) {\tikz\draw[step=0.25] (0,0) grid (0.5,-1.5);};
	\node[fill=white,font=\large,inner sep=0.9pt] (G-classic-label) at (G-classic) {\H};
	\node[anchor=south,font=\small, xshift=7.5pt, yshift=10pt] (h-label) at (G-classic.north west) {$\H_{\bullet,l}^{i-1}$};
	\node[anchor=south,font=\small, xshift=7.5pt, yshift=-5pt] (h-label-anchor) at (G-classic.north west) {};
	\node[right = of G-classic] (G-classic-updated) {\tikz\draw[step=0.25] (0,0) grid (0.5,-1.5);};
	\node[fill=blue!40!white, rotate=90, xshift=-7.5pt, yshift=-7.25pt] (color-box) at (G-classic-updated.north west) {};
	\node[anchor=south,font=\small, xshift=7.5pt, yshift=10pt] (h-2-label) at (G-classic-updated.north west) {$H_{l,l}^i$};
	\node[anchor=south,font=\small, xshift=7.5pt, yshift=-15pt] (h-2-label-anchor) at (G-classic-updated.north west) {};

	\node[fill=white,font=\large, inner sep=0.9pt] (G-classic-label) at (G-classic-updated) {\H};
	\draw[->, semithick] (G-classic.east) -- (G-classic-updated.west);
	
	\draw[->, semithick,yshift=-15pt, xshift=-15pt] (k-label.east) -- (k-label-anchor.west);
	
	\draw[->, semithick,yshift=-15pt, xshift=-15pt] (h-label.south) -- (h-label-anchor.north);
	
	\draw[->, semithick,yshift=-15pt, xshift=-15pt] (h-2-label.south) -- (h-2-label-anchor.north);
	\end{tikzpicture}%
	
	\caption{Eigenvector update using the feature-by-feature covariance matrix.}
	\label{fig:lin-eq}
\end{figure}
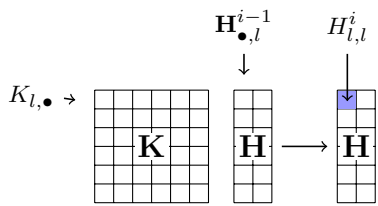

\begin{proposition}\label{it-leak:approach}
Let $\D \in \mathbb{R}^{n\times m}$ be the data matrix and denote $\K=\D^\top\D$ the feature-by-feature covariance matrix, which is unknown to the aggregator. Let $k$ be the number of eigenvectors retrieved.  When applying federated subspace iteration, the aggregator can reconstruct $\K$ after $m/k$ distinct eigenvector updates by solving a system of linear equations of the form $\K_{l,\bullet}\mathbf{A}=\mathbf{b}$ for each row $\K_{l,\bullet}$ of \K, where $\mathbf{A}\in\mathbb{R}^{m\times m}$ and $\mathbf{b}\in\mathbb{R}^{m}$ are known parameters. 
\end{proposition}

\begin{proof}
Let $\K_{l,\bullet}$ denote a row of the covariance matrix $\K\in\mathbb{R}^{m\times m}$. First, we show how series $(\H^i_{\bullet, 1})_{i=1}^m$ of $m$ updates of the first eigenvector can be used to retrieve the row $\K_{l,\bullet}$ of $\K$. Susequently, we show that $m/k$ updates are sufficient if all $k$ eigenvectors are used. 

Since $\K_{l,\bullet}$ is a row vector of length $m$, one needs $m$ equations, which can be derived from $m$ consecutive updates of the column vector $\H^i_{\bullet, 1}$. The aggregator can store the consecutive updates of $\H^i_{\bullet, 1}$ and, for each $i$, store an equation of the form $\K_{l,\bullet} \H^{i-1}_{\bullet, 1}=H^i_{l,1}$. After $m$ iterations, the aggregator is able to formulate the following fully determined system of linear equations, given that the eigenvectors have not converged:

\begin{equation*}
    \K_{l,\bullet}\begin{bmatrix}
    \H^0_{\bullet,1} & \cdots & \H^{m-1}_{\bullet,1}
    \end{bmatrix}=\begin{bmatrix}
    H^1_{l,1} & \cdots & H^m_{l,1}
    \end{bmatrix}
\end{equation*}


In order to reduce the number of required iterations, the aggregator can use all vectors in $\H$ to formulate the linear system and thereby divide the number of required iterations by $k$: 

\begin{equation*}
    \K_{l,\bullet}\overbrace{\begin{bmatrix}
    \H^0_{\bullet,1} & \cdots & \H^0_{\bullet,k} & \cdots & \H^{\frac{m}{k}-1}_{\bullet,1} & \cdots & \H^{\frac{m}{k}-1}_{\bullet,k}
    \end{bmatrix}}^{\mathbf{A}}=\overbrace{\begin{bmatrix}
    H^1_{l,1} & \cdots & H^1_{l,k} & \cdots & H^{\frac{m}{k}}_{l,1} & \cdots & H^{\frac{m}{k}}_{l,k}
    \end{bmatrix}}^{\mathbf{b}}
\end{equation*}


The rows of $\K$ can be computed simultaneously, by forming a system for all $\K_{l, \bullet}$ at the same time. Therefore, in theory, this means that after $m/k$ iterations, one has the full system and can solve it as
\begin{equation}
	\K_{l, \bullet}=\mathbf{b}\mathbf{A}^{-1}\text{,}
	\label{eq:naive-solve}
\end{equation}
which completes the proof of the proposition.
\end{proof}

These theoretical results require to invert $\mathbf{A}$, which may pose a problem in numerical applications, especially once the $\H^i$ grow large. By using a linear least squares solver, the inversion of the matrix $\mathbf{A}$ can be prevented at the cost of possibly sub-optimal solutions. Furthermore, in practice, more care needs to be taken when constructing the system, because, once converged, the eigenvectors do not provide a new equation to be added to the system anymore and hence lead to a singular system. In  \Cref{it-leak:numerical}, we show that our approach works on small data.

\subsection{Mitigation strategies}\label{sec:leakage:solution}

Recall that we claimed that \Cref{alg:approx-smpc,alg:approx-random} improve the privacy of subspace iteration. After having established that the full global covariance matrix can be reconstructed after sufficiently many iterations, it becomes clear that reducing the number of iterations makes this attack more difficult. \Cref{alg:approx-smpc} achieves this by using a better initial eigenvector guess and thus reduces the number of iteration until convergence. The randomized \Cref{alg:approx-random} shares the initial eigenvector updates, but then shares only proxy eigenvectors, whose entries do not correspond to real features in the data, effectively reducing the number of useful iterations for the aggregator to a constant number $I^\prime$. Therefore, these algorithms provide a algorithmic privacy improvement over previous solutions. 

The attack approach described above is possible even when secure multiparty computation (SMPC) \citep{cramer2015secure} is used, as the aggregated updates still become available in clear text at the aggregator. SMPC does however prevent the disclosure of the local covariance matrices, so using it is beneficial in truly federated implementations of this algorithm. Apart from the approximate algorithm, which uses SVD as an aggregation strategy, all algorithms are trivially compatible with secure aggregation as employed according to \cite{cramer2015secure}. Naturally, perturbation techniques like differential privacy \citep{Balcan2016} can be used to prevent the presented attack at the cost of decreased result accuracy. However, the high dimensionality $m$ of the data might prove prohibitive, as the noise scales with $m$.

%% file: experiments.tex

\section{Empirical evaluation}\label{sec:eval}

\subsection{Test datasets}
To evaluate our federated PCA algorithm, we used three publicly available datasets: chromosome 1 and 2 from a genetic dataset from the 1000 Genomes Project \citep{Auton2015}, as well as the MNIST database of handwritten digits \citep{mnist} (\Cref{tab:dataset-overview}). The two genetic data sets contain data for \num{2502} individuals (samples). After applying standard pre-processing steps (MAF filtering, LD pruning) we created \num{3} data set versions for each chromosome, with \num{100000}, \num{500000} and $>$\num{1000000} SNPs, respectively. MNIST contains \num{60000} grayscale images of handwritten numerals (samples), each of which has \num{784} pixels (features). This data set was split into \num{5} and \num{10} equal chunks. To the best of our knowledge, publicly available genetic data sets with large numbers of patients are not readily available. However, although motivated by federated GWAS, our federated SVD algorithm is actually generically applicable. The experiments on MNIST demonstrate its usefulness for a more general audience. The MNIST data set is available at \url{http://yann.lecun.com/exdb/mnist/}, the genetic data can be obtained from \url{ftp://ftp.1000genomes.ebi.ac.uk/vol1/ftp/release/20130502/}.

 \begin{table}[h]
\centering
\caption{Datasets used in the study.}
    \label{tab:dataset-overview}
    \begin{tabular}{lS[table-format=6.0]S[table-format=6.0]}
        \toprule
        Dataset & {Samples} & {Features} \\
        \midrule
         MNIST & 60000 & 784   \\
         1000 Genomes -- Chrom.\ 1 & 2502 & 100000  \\
         1000 Genomes -- Chrom.\ 1 & 2502 & 500000  \\
         1000 Genomes -- Chrom.\ 1 & 2502 & 1069419  \\
         1000 Genomes -- Chrom.\ 2 & 2502  & 100000 \\
         1000 Genomes -- Chrom.\ 2 & 2502  & 500000 \\
         1000 Genomes -- Chrom.\ 2 & 2502  & 1140556 \\
        \bottomrule
    \end{tabular}
\end{table}

\subsection{Compared methods}
We compare several configurations against each other: Federated subspace iteration with random initialization (\FEDRI), federated subspace iteration with approximate initialization (\FEDAI), and federated randomized subspace iteration (\FEDRANDRI). Federated subspace iteration which employs federated orthonormalization in every round (\FEDGS) is extremely communication inefficient by adding $2k$ additional rounds per iteration which has been proven a major bottleneck in preliminary studies. Therefore, we omit this algorithm from the empirical evaluation, because the practical use for SVD is limited. We compare ourselves to the algorithm presented by \cite{Guo2012}, denoted \GUO, as they present a solution which omits the covariance matrix and a way to deal with vertical data partitioning. However, \GUO shares the right singular vectors $\G$ and all updates of the left singular vectors $\H$ with the aggregator, which should be avoided in federated GWAS as emphasized in \Cref{sec:algo} and \Cref{sec:leakage}. Furthermore, as the number of features grows large, the transmission cost increases. We tested other configurations, including the use of approximate initialization for randomized PCA, but excluded them in this article as they did not bring a gain in performance in practice. For all compared methods, we set the convergence criterion in \cref{eq:convergence-power} to $\epsilon=10^{-9}$, which corresponds to a change of the angle between two consecutive eigenvectors updates of about $0.0026$ degrees. Note that this angle does not equal the angle \wrt centrally computed eigenvectors, which we used as a test metric for measuring the quality of the compared methods (\cf next subsection).

\begin{figure*}[htb]
\centering
\includegraphics[width=\linewidth]{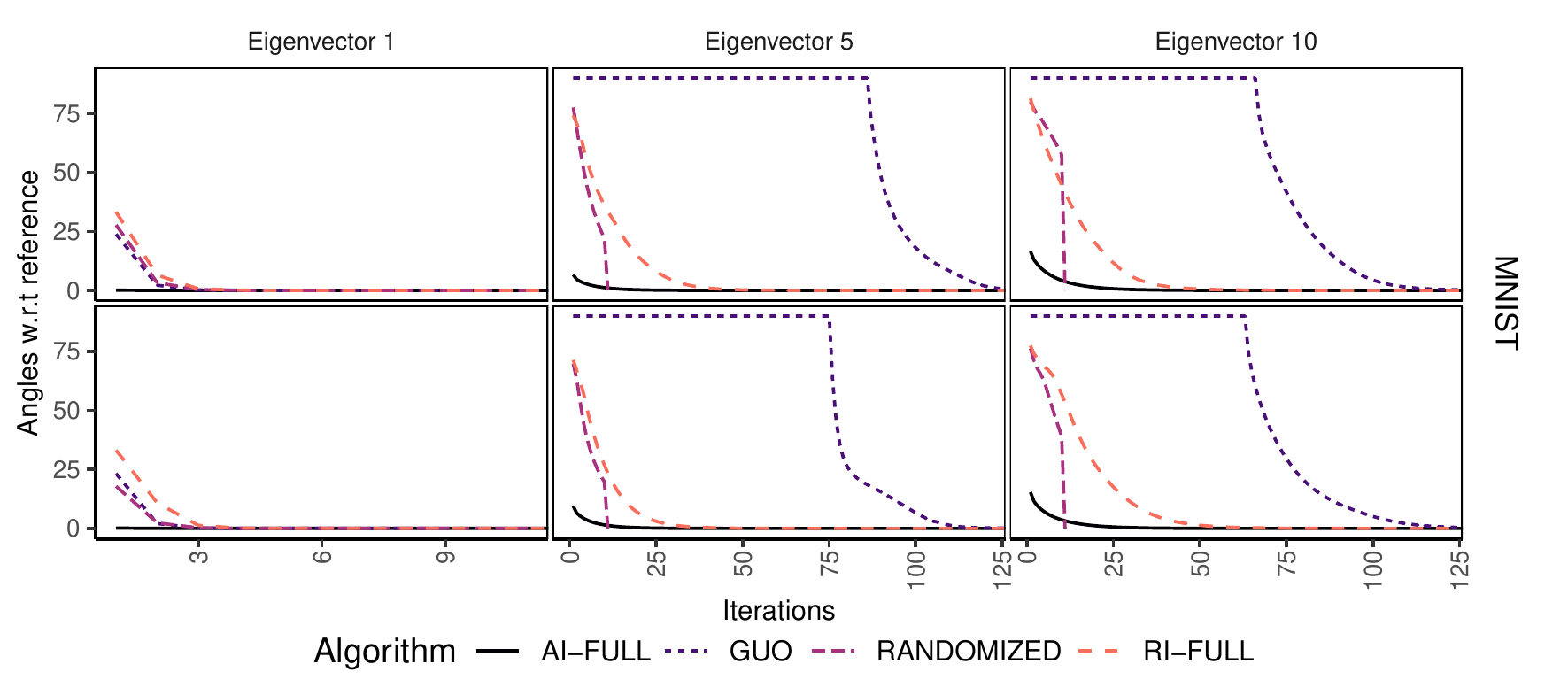}
\caption{Angles between selected reference eigenvectors and the federated eigenvectors for the MNIST data. The omitted eigenvectors show similar behaviors.}
\label{fig:angles-mnist}
\end{figure*}

\begin{figure*}[htb]
	\centering
	\includegraphics[width=\linewidth]{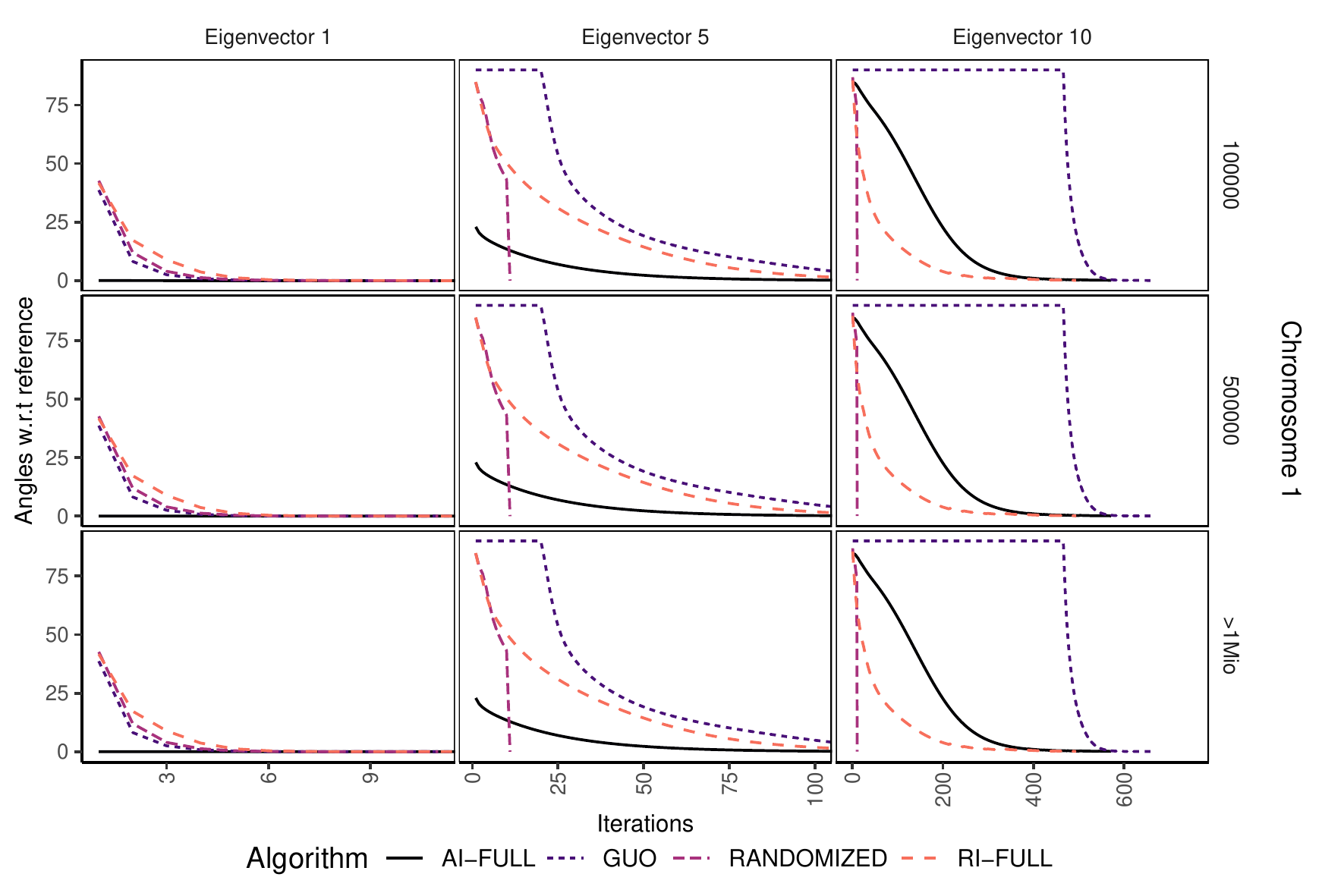}
	\includegraphics[width=\linewidth]{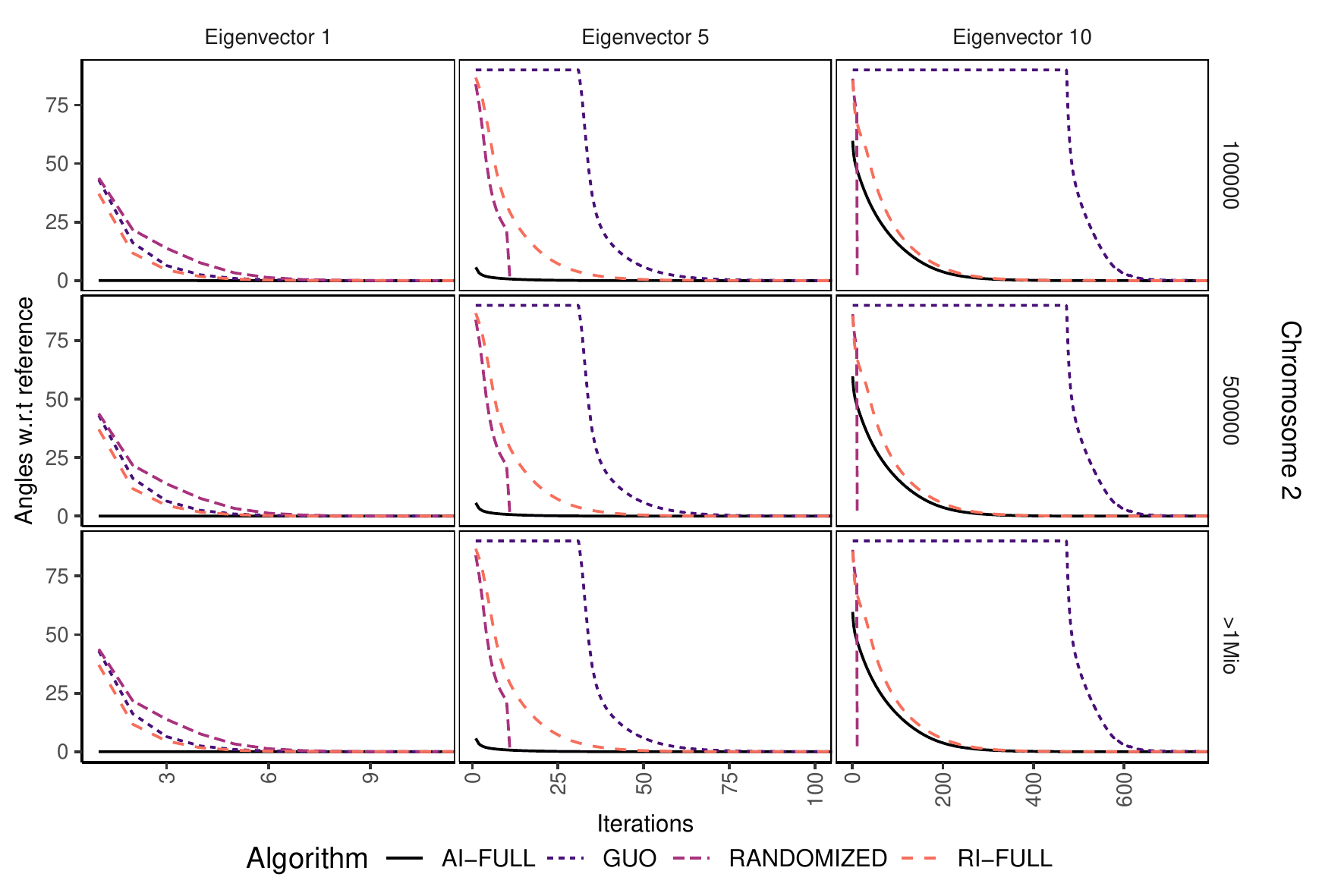}
	\caption{Angles between selected reference eigenvectors and the federated eigenvectors on chromosome 1 and 2. The omitted eigenvectors show similar behaviors.}
	\label{fig:angles-chr1}
\end{figure*}

\begin{figure}[htb]
	\centering
	\includegraphics[width=\linewidth]{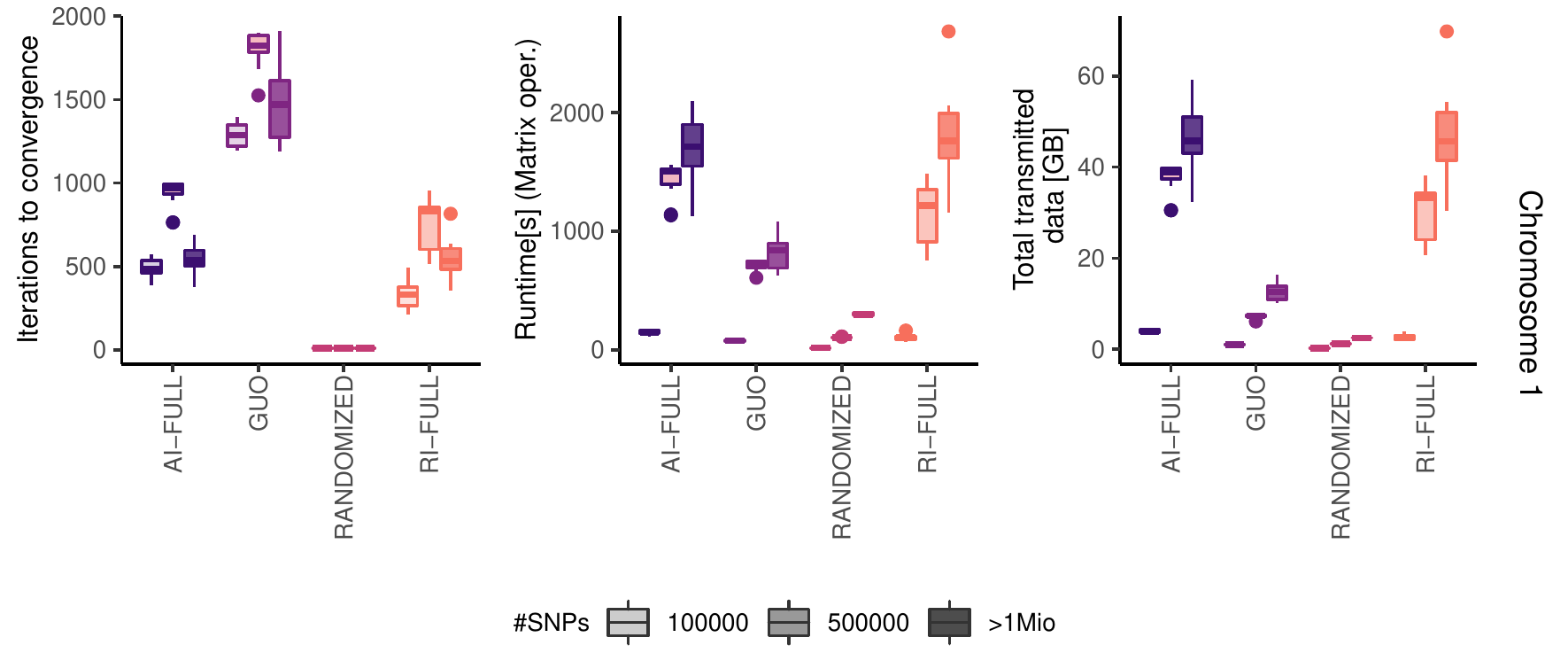}
	\includegraphics[width=\linewidth]{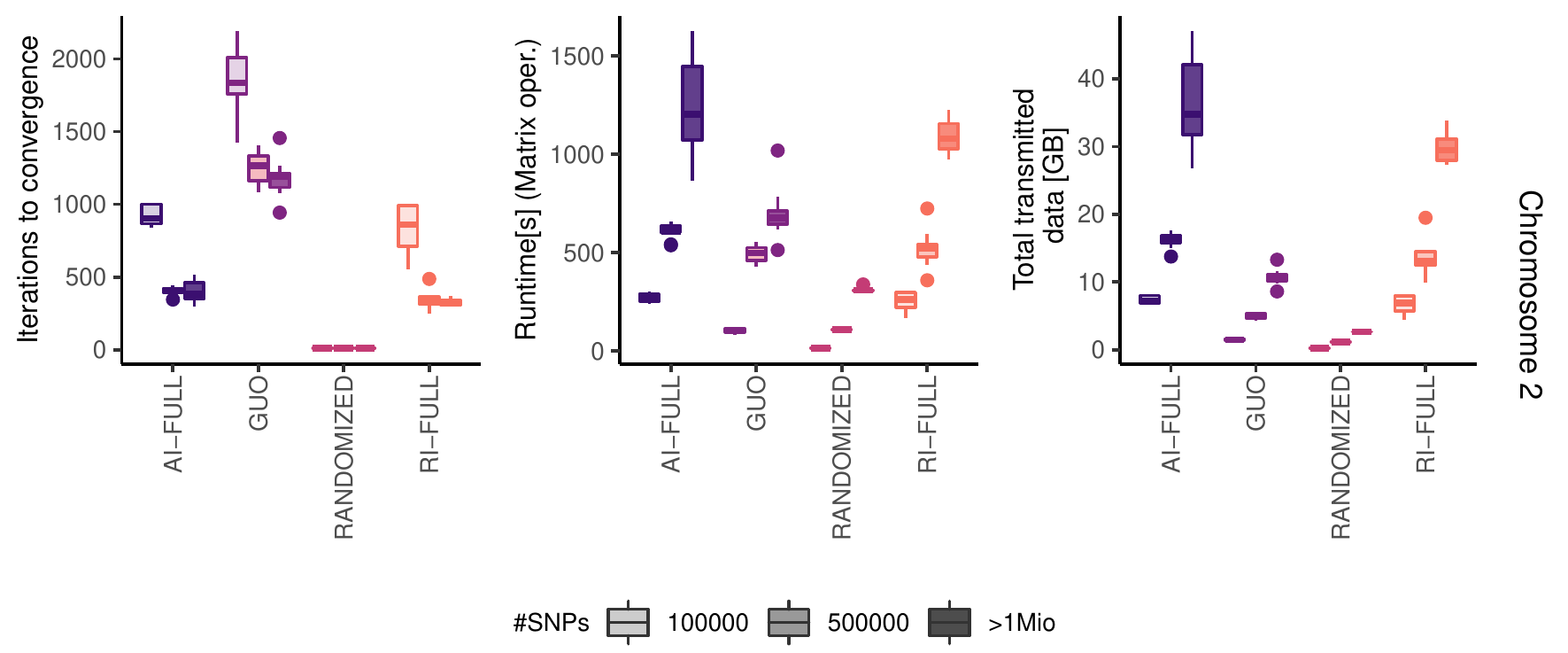}
	\caption{Iterations, Runtime for matrix computations, and total transmitted data for each algorithm and each of the three data sets. The boxplots are grouped, the shading indicates the number of features ($0.1*10^6$, $0.5*10^6$, and roughly $1*10^6$).}
	\label{fig:benchmark}
\end{figure}

\begin{figure}[htb]
	\centering
	\includegraphics[width=\linewidth]{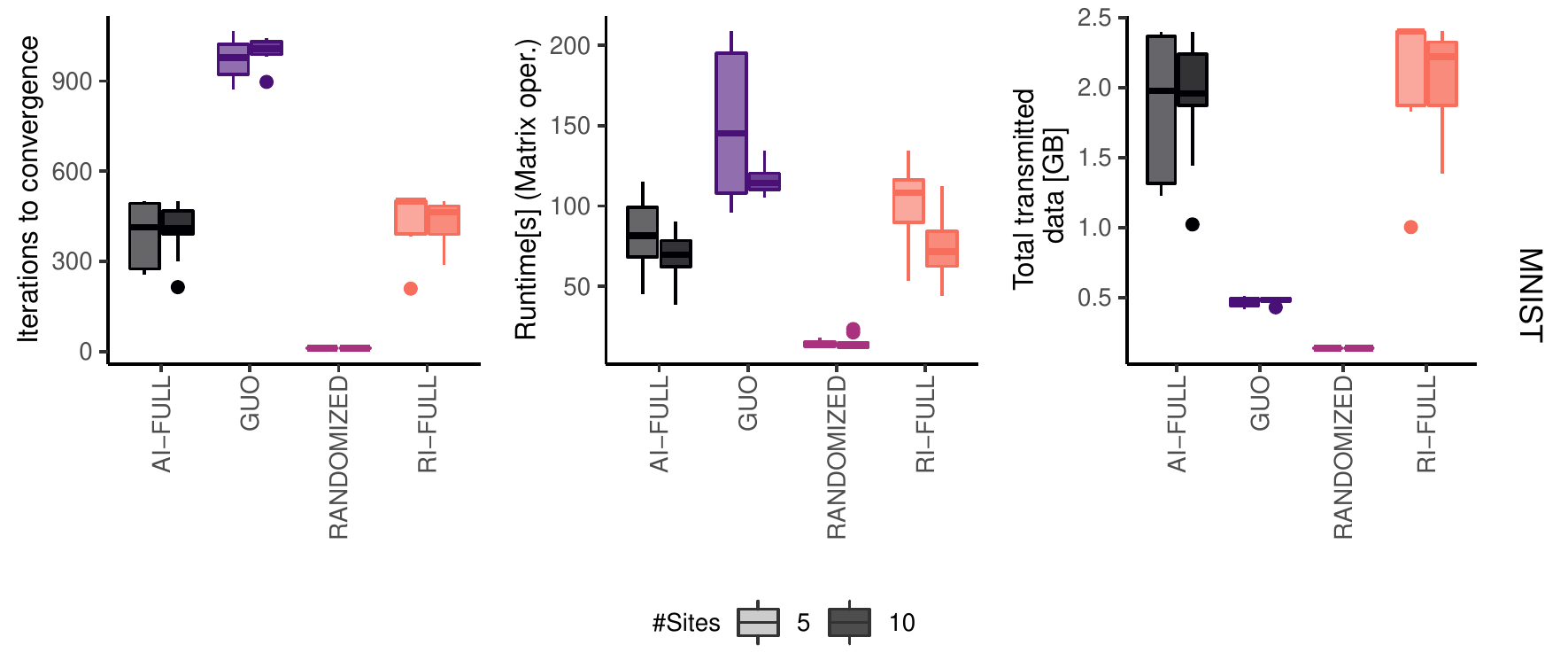}
	\caption{Iterations, Runtime for matrix computations, and total transmitted data for each algorithm for the MNIST data. The boxplots are grouped, the shading indicates the number of simulated clients (5, and 10).}
	\label{fig:benchmark-mnist}
\end{figure}

\subsection{Test metrics}
For measuring the quality of the compared methods, we computed the angles between the eigenvectors obtained from a reference implementation of a centralized PCA and their counterparts computed in a federated fashion. An angle of $0$ between two eigenvectors of the same rank is the desired result. As a reference, we chose the version implemented in \texttt{scipy.sparse.linalg}, which internally interfaces LAPACK. The amount of transmitted data is estimated by calculating the number of transmitted floats and multiplying it by a factor of 4 bytes (single precision IEEE 754). We choose this metric to remain agnostic with respect to the transmission protocol. Times measures are wall clock times using Python's \texttt{time} module. We chose to measure the runtime for matrix operations only, as they are the most important contributor to the overall runtime apart from communication related runtime. 



\subsection{Implementation, availability, and hardware specifications}
All methods except the web interface are written in Python, using mainly, but not exclusively \texttt{numpy} and \texttt{scipy}. They are available online at {\url{https://github.com/AnneHartebrodt/federated-svd}}. The simulation tests were run on a compute server with 48 CPUs and 502 GB available RAM due to the size of the genetic data sets. A federated tool compatible with the FeatureCloud \citep{Matschinske2021} ecosystem (\url{featurecloud.ai}) is available on the platform. The corresponding source code can be found at {\url{https://github.com/AnneHartebrodt/fc-federated-svd}}. We also created an AIME report \citep{matschinske2021aime} to promote accessibility of machine learning research \citep{aimeReportlP0kqT}.

\subsection{Convergence behavior}
To test the convergence behavior of the compared federated algorithms, we split the genetic data sets into \num{5} equally sized chunks; and the MNIST data set into \num{5} and \num{10} chunks.  For every algorithm, we then recorded the angles between the first 10 eigenvectors \wrt the fully converged references at each iteration averaged across \num{10} runs.

\Crefrange{fig:angles-mnist}{fig:benchmark-mnist} show the results of the experiments. Note that, unlike the versions \FEDRI, \FEDAI, \FEDRANDRI of our algorithm, the competitor \GUO computes the eigenvectors sequentially (\ie, eigenvector $k$ has to converge before starting the computation of the eigenvector $k+1$), which means that, for all but the first eigenvector, the plots for \GUO start with a horizontal line. The most important result is that, for all algorithms, the eigenvectors perfectly converge to the reference eventually. 

For low ranking eigenvectors, the approximate initialization speeds up the computation, because these eigenvectors can be well approximated and start with angles to the reference close to 0 (see \Cref{fig:angles-mnist,fig:angles-chr1}). The gain decreases in higher dimensions. Therefore, \FEDRANDRI shows the overall best convergence behavior across all data sets and dimensions. 

The number of sites does not influence the convergence behavior, as can be seen in the test with the MNIST data (\Cref{fig:angles-mnist,fig:benchmark-mnist}) where the convergence curves and the required number of iterations are similar for the simulations with \num{5} clients and \num{10} clients. The transmitted data is shown only from the aggregator's perspective, to make the runs comparable. In federated SVD, the transmission cost scales with the number of clients.

The number of features/SNPs in the data does not show a clear trend. Although in the convergence plots in \Cref{fig:angles-chr1} the larger data sets seem to converge more quickly, the overall number of iterations in \cref{fig:benchmark} does not confirm this trend. The reason for this is the dependence of the convergence speed on the eigengaps (the difference between two consecutive eigenvalues), an inherent property of each data set. The smaller the eigengap, the worse the convergence behavior. \Cref{tab:dataset-eigengap} shows the eigengaps for the eigengaps for Chromosome 2. The higher ranking eigengaps are generally quite small, indicating generally bad convergence for all datasets. Eigengap 8 for the data set containing 100000 SNPs specifically, is comparably even smaller which could explain the especially poor convergence.

 \begin{table}[h]
\centering
\caption{Eigengaps for Chromosome 2}
    \label{tab:dataset-eigengap}
    \begin{tabular}{lSSSSSSSSSS}
        \toprule
        {SNPs} & {EG1} & {EG2}& {EG3}& {EG4}& {EG5}& {EG6}& {EG7}& {EG8}& {EG9} \\
        \midrule
        100000 & 15.14&16.02&1.19&3.95&0.9&1.87&0.32&0.05&0.12  \\
        500000 &  26.68&18.41&3.36&12.65&2.71&0.62&0.91&0.2&0.28 \\
        1140556 &  31.5&24.79&4.02&15.59&3.3&3.41&0.39&1.2&0.25 \\
        \bottomrule
    \end{tabular}
\end{table}

\subsection{Scalability}
To gauge the scalability of the methods with respect to runtime and transmission cost, we recorded the number of iterations until convergence, the runtime for matrix operations, and the estimated total amount of transmitted data for the selected algorithms. In \Cref{fig:benchmark} we see that the amount of transmitted data is the smallest for the randomized algorithm \FEDRANDRI, followed by \GUO and with a significant distance \FEDAI and \FEDRI. \FEDRANDRI also spends the least time on the matrix operations which are the major contributor in to the runtime. The number of required iterations is the smallest for \FEDRANDRI, followed with large gap by \FEDAI and \FEDRI and \GUO on the last place. Since the required iterations correspond to the number of communication steps, this factor significantly contributes to the overall runtime. Overall, \FEDRANDRI performs the best in all three measured categories. Generally, the bottleneck in federated learning is the number of transmission steps during the learning process, as this involves network communication. However, with increasing data set size like in the presented GWAS case, also the reduction in local runtime shows considerable impact on the overall runtime.

\subsection{Covariance reconstruction experiment}
\label{it-leak:numerical}

We implemented the covariance reconstruction scheme presented in \Cref{sec:leakage} and applied in on small example data, demonstrating its practicality. Using the breast cancer data and the diabetes data set from the UCI repository \citep{Dua2019} which have \num{442} and \num{569} samples and \num{10} and {30} features respectively, we computed the centralized covariance matrix. Then we ran federated subspace iteration and recorded the eigenvector updates. After $m/k$ iterations, we used the recorded matrices to form the linear system described in \Cref{sec:leakage:problem}. Instead of inverting the matrix as described in \cref{eq:naive-solve}, we used a linear least squares solver (\texttt{scipy.lstsq}) to compute the solution. We then computed the Pearson correlation between the true and the reconstructed covariance matrix, with a perfect outcome of 1 (see \Cref{tab:dataset-reconstruction}) in negligible time. 

 \begin{table}[ht]
	\centering
	\caption{Reconstruction of covariance matrix.}
	\label{tab:dataset-reconstruction}
	\begin{tabular}{lcccc}
		\toprule
		Dataset & {Samples} & {Features} & {Correlation} & {Time[s]} \\
		\midrule
		Breast Cancer & 442 & 10  & 1& 0.001\\
        Diabetes & 569 & 30 & 0.997 & 0.004 \\
		\bottomrule
	\end{tabular}
\end{table}

%% file: conclusions.tex

\section{Conclusions and outlook}\label{sec:conc}
In this paper, we presented an improved federated SVD algorithm which is applicable to both vertically and horizontally partitioned data and, at the same time, increases the privacy compared to previous solutions.

Although our algorithm is motivated by the requirements of population stratification in federated GWAS, it is generically applicable. We proved that a first version of our algorithm is equivalent to a state-of-the-art centralized SVD algorithm and demonstrated empirically that it indeed converges to the centrally computed solutions. Subsequently, we improved the algorithm by including techniques from other federated and centralized algorithms to increase scalability and reduce the number of required communications.

There are two key advantages of our algorithm: Firstly, unlike in existing federated PCA algorithms, the sample eigenvectors remain at the local sites, due to the use of fully federated Gram-Schmidt orthonormalization, which improves the privacy of the algorithm. Secondly, the algorithm limits the amount of transmitted data (via smart initialization and data approximation) and is thereby more scalable and further prevents information leakage. In particular, the transmission cost of the randomized algorithm is not dependent on the number of samples and only partially dependent on the number of features.

%% file: vertical-main.bbl
\begin{thebibliography}{41}
\providecommand{\natexlab}[1]{#1}
\providecommand{\url}[1]{{#1}}
\providecommand{\urlprefix}{URL }
\providecommand{\doi}[1]{\url{https://doi.org/#1}}
\providecommand{\eprint}[2][]{\url{#2}}
 \bibcommenthead

\bibitem[{Balcan et~al(2014)Balcan, Kanchanapally, Liang, and
  Woodruff}]{Balcan_2014}
Balcan MF, Kanchanapally V, Liang Y, et~al (2014) {Improved distributed
  principal component analysis}. Advances in Neural Information Processing
  Systems 4(January):3113--3121.
  \urlprefix\url{http://arxiv.org/abs/1408.5823},
  {\href{https://arxiv.org/abs/1408.5823}{{https://arxiv.org/abs/arXiv:1408.5823}}}

\bibitem[{Balcan et~al(2016)Balcan, Du, Wang, and Yu}]{Balcan2016}
Balcan MF, Du SS, Wang Y, et~al (2016) {An improved gap-dependency analysis of
  the noisy power method}. Journal of Machine Learning Research
  49(June):284--309. \urlprefix\url{http://arxiv.org/abs/1602.07046},
  {\href{https://arxiv.org/abs/1602.07046}{{https://arxiv.org/abs/arXiv:1602.07046}}}

\bibitem[{Chen et~al(2020)Chen, Lee, Li, and Yang}]{Chen2020}
Chen X, Lee JD, Li H, et~al (2020) Distributed estimation for principal
  component analysis: a gap-free approach. CoRR abs/2004.02336.
  {\href{https://arxiv.org/abs/2004.02336}{{https://arxiv.org/abs/arXiv:2004.02336}}}

\bibitem[{Cho et~al(2018)Cho, Wu, and Berger}]{Cho2018}
Cho H, Wu DJ, Berger B (2018) {Secure genome-wide association analysis using
  multiparty computation}. Nature Biotechnology 36(6):547--551.
  \doi{10.1038/nbt.4108}

\bibitem[{Cramer et~al(2015)Cramer, Damg{\aa}rd et~al}]{cramer2015secure}
Cramer R, Damg{\aa}rd IB, et~al (2015) Secure multiparty computation. Cambridge
  University Press

\bibitem[{Dua and Graff(2017)}]{Dua2019}
Dua D, Graff C (2017) {UCI} machine learning repository.
  \urlprefix\url{http://archive.ics.uci.edu/ml}

\bibitem[{Galinsky et~al(2016)Galinsky, Bhatia, Loh, Georgiev, Mukherjee,
  Patterson, and Price}]{Galinsky2016}
Galinsky KJ, Bhatia G, Loh PR, et~al (2016) {Fast Principal-Component Analysis
  Reveals Convergent Evolution of ADH1B in Europe and East Asia}. The American
  Journal of Human Genetics 98(3):456--472. \doi{10.1016/j.ajhg.2015.12.022}

\bibitem[{Gauch et~al(2019)Gauch, Qian, Piepho, Zhou, and Chen}]{Gauch2019}
Gauch HG, Qian S, Piepho HP, et~al (2019) {Consequences of PCA graphs, SNP
  codings, and PCA variants for elucidating population structure}. PLoS ONE
  14(6):1--26. \doi{10.1371/journal.pone.0218306}

\bibitem[{Grammenos et~al(2020)Grammenos, Mendoza~Smith, Crowcroft, and
  Mascolo}]{Grammenos2020}
Grammenos A, Mendoza~Smith R, Crowcroft J, et~al (2020) Federated principal
  component analysis. In: Larochelle H, Ranzato M, Hadsell R, et~al (eds)
  Advances in Neural Information Processing Systems, vol~33. Curran Associates,
  Inc., pp 6453--6464,
  \urlprefix\url{https://proceedings.neurips.cc/paper/2020/file/47a658229eb2368a99f1d032c8848542-Paper.pdf}

\bibitem[{Guo et~al(2012)Guo, Lin, Teng, Xue, and Fan}]{Guo2012}
Guo YF, Lin X, Teng Z, et~al (2012) {A covariance-free iterative algorithm for
  distributed principal component analysis on vertically partitioned data}.
  Pattern Recognition 45(3):1211--1219. \doi{10.1016/j.patcog.2011.09.002}

\bibitem[{{Hadri} et~al(2010){Hadri}, {Ltaief}, {Agullo}, and
  {Dongarra}}]{HadriQR}
{Hadri} B, {Ltaief} H, {Agullo} E, et~al (2010) Tile qr factorization with
  parallel panel processing for multicore architectures. In: 2010 IEEE
  International Symposium on Parallel Distributed Processing (IPDPS), pp 1--10,
  \doi{10.1109/IPDPS.2010.5470443}

\bibitem[{Halko et~al(2011)Halko, Martinsson, and Tropp}]{Halko_2010}
Halko N, Martinsson PG, Tropp JA (2011) Finding structure with randomness:
  Probabilistic algorithms for constructing approximate matrix decompositions.
  SIAM Review 53(2):217–288. \doi{10.1137/090771806}

\bibitem[{Hartbrodt(2022)}]{aimeReportlP0kqT}
Hartbrodt A (2022) Federated singular value decomposition for high dimensional
  data [aime lp0kqt]. \urlprefix\url{https://aime.report/lP0kqT}

\bibitem[{Hartebrodt et~al(2021)Hartebrodt, Nasirigerdeh, Blumenthal, and
  R\"ottger}]{Hartebrodt2021}
Hartebrodt A, Nasirigerdeh R, Blumenthal DB, et~al (2021) {Federated Principal
  Component Analysis for Genome-Wide Association Studies}. ICDM 2021

\bibitem[{{Hoemmen}(2011)}]{HommenQR}
{Hoemmen} M (2011) A communication-avoiding, hybrid-parallel, rank-revealing
  orthogonalization method. In: 2011 IEEE International Parallel Distributed
  Processing Symposium, pp 966--977, \doi{10.1109/IPDPS.2011.93}

\bibitem[{Imtiaz and Sarwate(2018)}]{Imtiaz2018}
Imtiaz H, Sarwate AD (2018) Differentially private distributed principal
  component analysis. In: 2018 IEEE International Conference on Acoustics,
  Speech and Signal Processing (ICASSP). IEEE, p 2206–2210,
  \doi{10.1109/ICASSP.2018.8462519}

\bibitem[{Jolliffe(2002)}]{Jolliffe2002}
Jolliffe I (2002) Principal Component Analysis. Springer-Verlag,
  \doi{10.1007/b98835}, \urlprefix\url{https://doi.org/10.1007/b98835}

\bibitem[{Kairouz et~al(2021)Kairouz, McMahan, Avent, Bellet, Bennis, Bhagoji,
  Bonawitz, Charles, Cormode, Cummings, D'Oliveira, Eichner, {El Rouayheb},
  Evans, Gardner, Garrett, Gasc{\'{o}}n, Ghazi, Gibbons, Gruteser, Harchaoui,
  He, He, Huo, Hutchinson, Hsu, Jaggi, Javidi, Joshi, Khodak, Konecn{\'{i}},
  Korolova, Koushanfar, Koyejo, Lepoint, Liu, Mittal, Mohri, Nock,
  {\"{O}}zg{\"{u}}r, Pagh, Qi, Ramage, Raskar, Raykova, Song, Song, Stich, Sun,
  Suresh, Tram{\`{e}}r, Vepakomma, Wang, Xiong, Xu, Yang, Yu, Yu, and
  Zhao}]{Kairouz2021}
Kairouz P, McMahan HB, Avent B, et~al (2021) {Advances and open problems in
  federated learning}. Foundations and Trends in Machine Learning
  14(1-2):1--210. \doi{10.1561/2200000083},
  {\href{https://arxiv.org/abs/1912.04977}{{https://arxiv.org/abs/arXiv:1912.04977}}}

\bibitem[{Kargupta et~al(2001)Kargupta, Huang, Sivakumar, and
  Johnson}]{Kargupta2001}
Kargupta H, Huang W, Sivakumar K, et~al (2001) {Distributed Clustering Using
  Collective Principal Component Analysis}. Knowledge and Information Systems
  \doi{10.4324/9781315799476-12}

\bibitem[{LeCun et~al(2005)LeCun, Cortes, and Burges}]{mnist}
LeCun Y, Cortes C, Burges CJ (2005) {MNNIST database of handwritten digits}.
  \url{http://yann.lecun.com/exdb/mnist/}, [Online; accessed 27-02-2020]

\bibitem[{Lei et~al(2016)Lei, Zhong, and Dhillon}]{Lei2016}
Lei Q, Zhong K, Dhillon IS (2016) Coordinate-wise power method. In: Lee D,
  Sugiyama M, Luxburg U, et~al (eds) Advances in Neural Information Processing
  Systems, vol~29. Curran Associates, Inc., p 2064–2072,
  \urlprefix\url{https://proceedings.neurips.cc/paper/2016/file/8b4066554730ddfaa0266346bdc1b202-Paper.pdf}

\bibitem[{Li et~al(2016)Li, Byun, Cai, Xiao, Han, Cornelis, Dinulos, Dennis,
  Easton, Gorlov, Seldin, and Amos}]{Li2016}
Li Y, Byun J, Cai G, et~al (2016) {FastPop: A rapid principal component derived
  method to infer intercontinental ancestry using genetic data}. BMC
  Bioinformatics 17(1):1--8. \doi{10.1186/s12859-016-0965-1}

\bibitem[{Liu et~al(2020)Liu, Chen, Zheng, Wang, Zhou, Liu, and Yang}]{Liu2020}
Liu Y, Chen C, Zheng L, et~al (2020) {Privacy Preserving PCA for Multiparty
  Modeling}. arXiv \urlprefix\url{http://arxiv.org/abs/2002.02091},
  {\href{https://arxiv.org/abs/2002.02091}{{https://arxiv.org/abs/2002.02091}}}

\bibitem[{Londin et~al(2010)Londin, Keller, Maista, Smith, Mamounas, Zhang,
  Madore, Gwinn, and Corriveau}]{Londin2010}
Londin ER, Keller MA, Maista C, et~al (2010) Coaims: A cost-effective panel of
  ancestry informative markers for determining continental origins. PLoS One
  5:e13,443. \doi{10.1371/journal.pone.0013443}

\bibitem[{Matschinske et~al(2021{\natexlab{a}})Matschinske, Alcaraz, Benis,
  Golebiewski, Grimm, Heumos, Kacprowski, Lazareva, List, Louadi, Pauling,
  Pfeifer, R{\"o}ttger, Schw{\"a}mmle, Sturm, Traverso, Van~Steen, de~Freitas,
  Villalba~Silva, Wee, Wenke, Zanin, Zolotareva, Baumbach, and
  Blumenthal}]{matschinske2021aime}
Matschinske J, Alcaraz N, Benis A, et~al (2021{\natexlab{a}}) The {AIMe}
  registry for artificial intelligence in biomedical research. Nature Methods
  \doi{10.1038/s41592-021-01241-0},
  \urlprefix\url{https://doi.org/10.1038/s41592-021-01241-0}

\bibitem[{Matschinske et~al(2021{\natexlab{b}})Matschinske, Späth,
  Nasirigerdeh, Torkzadehmahani, Hartebrodt, Orbán, Fejér, Zolotareva,
  Bakhtiari, Bihari, Bloice, Donner, Fdhila, Frisch, Hauschild, Heider,
  Holzinger, Hötzendorfer, Hospes, Kacprowski, Kastelitz, List, Mayer, Moga,
  Müller, Pustozerova, Röttger, Saranti, Schmidt, Tschohl, Wenke, and
  Baumbach}]{Matschinske2021}
Matschinske J, Späth J, Nasirigerdeh R, et~al (2021{\natexlab{b}}) The
  {FeatureCloud AI} store for federated learning in biomedicine and beyond.
  \eprint{2105.05734}

\bibitem[{Mothukuri et~al(2021)Mothukuri, Parizi, Pouriyeh, Huang,
  Dehghantanha, and Srivastava}]{Mothukuri_2021}
Mothukuri V, Parizi RM, Pouriyeh S, et~al (2021) A survey on security and
  privacy of federated learning. Future Generation Computer Systems
  115:619–640. \doi{10.1016/j.future.2020.10.007}

\bibitem[{Nasirigerdeh et~al(2020)Nasirigerdeh, Torkzadehmahani, Matschinske,
  Frisch, List, Sp{\"a}th, Wei{ss}, V{\"o}lker, Wenke, Kacprowski, and
  Baumbach}]{splink2020}
Nasirigerdeh R, Torkzadehmahani R, Matschinske J, et~al (2020) splink: A
  federated, privacy-preserving tool as a robust alternative to meta-analysis
  in genome-wide association studies. bioRxiv \doi{10.1101/2020.06.05.136382}

\bibitem[{Nasirigerdeh et~al(2021)Nasirigerdeh, Torkzadehmahani, Baumbach, and
  Blumenthal}]{sigir2021}
Nasirigerdeh R, Torkzadehmahani R, Baumbach J, et~al (2021) On the privacy of
  federated pipelines. In: SIGIR 2021. ACM, New York, p~5,
  \doi{10.1145/3404835.3462996}

\bibitem[{Price et~al(2006)Price, Patterson, Plenge, Weinblatt, Shadick, and
  Reich}]{Price2006}
Price AL, Patterson NJ, Plenge RM, et~al (2006) {Principal components analysis
  corrects for stratification in genome-wide association studies}. Nature
  Genetics 38(8):904--909. \doi{10.1038/ng1847}

\bibitem[{Qi et~al(2003)Qi, Wang, and Birdwell}]{Qi2003}
Qi H, Wang TW, Birdwell JD (2003) {Global principal component analysis for
  dimensionality reduction in distributed data mining}. In: Statistical Data
  Mining and Knowledge Discovery. Chapman and Hall/CRC, p 323--338,
  \doi{10.1201/9780203497159.ch19}

\bibitem[{Rodr{\'i}guez et~al(2017)Rodr{\'i}guez, Fern{\'a}ndez, Peregr{\'i}n,
  and Herrera}]{B2017}
Rodr{\'i}guez M{\'A}, Fern{\'a}ndez A, Peregr{\'i}n A, et~al (2017) A Review of
  Distributed Data Models for Learning. Springer International Publishing, Cham

\bibitem[{Saad(2011)}]{Saad2011}
Saad Y (2011) Numerical Methods for Large Eigenvalue Problems. Classics in
  Applied Mathematics, Society for Industrial and Applied Mathematics,
  \doi{10.1137/1.9781611970739}

\bibitem[{Sanchez-Fernandez et~al(2015)Sanchez-Fernandez, Fuente, and
  Sainz-Palmero}]{SanchezFernandez2015}
Sanchez-Fernandez A, Fuente M, Sainz-Palmero G (2015) Fault detection in
  wastewater treatment plants using distributed pca methods. In: 2015 IEEE 20th
  Conference on Emerging Technologies \& Factory Automation (ETFA). IEEE, p
  1–7, \doi{10.1109/ETFA.2015.7301504}

\bibitem[{Sluciak et~al(2016)Sluciak, Strakov{\'{a}}, Rupp, and
  Gansterer}]{Sluciak2016}
Sluciak O, Strakov{\'{a}} H, Rupp M, et~al (2016) {Distributed Gram-Schmidt
  orthogonalization with simultaneous elements refinement}. Eurasip Journal on
  Advances in Signal Processing 2016(1):1--13. \doi{10.1186/s13634-016-0322-6}

\bibitem[{Steed and Fradinho Duarte~de Oliveira(2010)}]{Steed2010}
Steed A, Fradinho Duarte~de Oliveira M (2010) More than two. Network Graphics
  pp 125--168. \doi{10.1016/B978-0-12-374423-4.00004-5}

\bibitem[{Strakov{\'a} et~al(2012)Strakov{\'a}, Gansterer, and
  Zemen}]{Strakov2012}
Strakov{\'a} H, Gansterer WN, Zemen T (2012) Distributed qr factorization based
  on randomized algorithms. In: Wyrzykowski R, Dongarra J, Karczewski K, et~al
  (eds) Parallel Processing and Applied Mathematics. Springer Berlin
  Heidelberg, Berlin, Heidelberg, pp 235--244

\bibitem[{Tam et~al(2019)Tam, Patel, Turcotte, Boss{\'{e}}, Par{\'{e}}, and
  Meyre}]{Tam2019}
Tam V, Patel N, Turcotte M, et~al (2019) {Benefits and limitations of
  genome-wide association studies}. Nature Reviews Genetics 20(8):467--484.
  \doi{10.1038/s41576-019-0127-1}

\bibitem[{{The 1000 Genomes Consortium, Auton, A.}(2015)}]{Auton2015}
{The 1000 Genomes Consortium, Auton, A.} (2015) {A global reference for human
  genetic variation}. Nature 526(7571):68--74. \doi{10.1038/nature15393}

\bibitem[{Visscher et~al(2017)Visscher, Wray, Zhang, Sklar, McCarthy, Brown,
  and Yang}]{Visscher2017}
Visscher PM, Wray NR, Zhang Q, et~al (2017) {10 Years of GWAS Discovery:
  Biology, Function, and Translation}. American Journal of Human Genetics
  101(1):5--22. \doi{10.1016/j.ajhg.2017.06.005}

\bibitem[{Wu et~al(2018)Wu, Wai, Li, and Scaglione}]{Wu2018}
Wu SX, Wai HT, Li L, et~al (2018) {A Review of Distributed Algorithms for
  Principal Component Analysis}. Proceedings of the IEEE 106(8):1321--1340.
  \doi{10.1109/JPROC.2018.2846568}

\end{thebibliography}
